\documentclass{article}
\usepackage[preprint]{arxiv}
\usepackage{microtype}
\usepackage{hyperref}
\usepackage{url}
\usepackage{booktabs}
\usepackage{amsmath}
\usepackage{amssymb}
\usepackage{amsfonts}
\usepackage{xcolor}
\usepackage{graphicx}
\usepackage{float}
\usepackage{caption}
\usepackage{multirow}
\usepackage{multirow}
\usepackage{booktabs}
\usepackage{amsthm}
\theoremstyle{definition}
\newtheorem{definition}{Definition}[section]
\usepackage{booktabs}
\usepackage{tabularx}
\usepackage{array}
\usepackage{adjustbox}
\newtheorem{proposition}{Proposition}
\usepackage{booktabs}
\usepackage{tabularx}
\usepackage{array}
\usepackage{adjustbox}
\usepackage[most]{tcolorbox}
\usepackage{listings}
\usepackage[T1]{fontenc}
\usepackage{textcomp}
\usepackage{float}
\usepackage{wrapfig}
 \usepackage{booktabs}
 \usepackage{caption}
 \usepackage[most]{tcolorbox}
 \usepackage[T1]{fontenc}
 \usepackage{textcomp}
 \usepackage{listings}
 \lstdefinestyle{promptstyle}{
   basicstyle=\ttfamily\tiny,
   breaklines=true,
   breakatwhitespace=false,
   columns=fullflexible,
   keepspaces=true,
   showstringspaces=false,
   upquote=true
 }
 \tcbset{
   promptbox/.style={
     enhanced,
     breakable,
     colback=gray!3,
     colframe=gray!35,
     coltitle=black,
     fonttitle=\bfseries,
     boxrule=0.35pt,
     arc=2pt,
     left=3pt,
     right=3pt,
     top=3pt,
     bottom=3pt,
     boxsep=2pt,
   }
 }
\usepackage{lineno}
\usepackage[utf8]{inputenc}
\usepackage[T1]{fontenc}
\usepackage{hyperref}
\usepackage{url}
\usepackage{booktabs}
\usepackage{amsfonts}
\usepackage{nicefrac}
\usepackage{microtype}
\usepackage{xcolor}
\usepackage{booktabs}
\usepackage{fancyvrb}
\usepackage{adjustbox}
\usepackage{xcolor}
\usepackage{setspace}
\newtheorem{corollary}{Corollary}
\usepackage{xcolor}
\title{Textual Belief States for World Models: Identifiable Representation Learning Under Strict Mediation}
\author{
  Xiang Gao, Kaiwen Dong, Yuguang Yao, Padmaja Jonnalagedda, Kamalika Das \\
  Intuit AI Research\\
  \texttt{\{xiang\_gao, kamalika\_das\}@intuit.com} \\
}
\begin{document}
\maketitle
\begin{abstract}
World models in partially observed environments rely on latent representations that summarize interaction history, but in many modern LLM-based architectures predictive performance fails to reflect representation quality due to history bypass, rendering the latent state unidentifiable. Strict latent state mediation, requiring predictions to depend only on the latent state and action, is a classical principle that resolves this, but enforcing it in text-based settings is an open challenge: textual latent states are discrete and non-differentiable, precluding variational training, and expressive LLM decoders readily ignore the bottleneck. We show how to make strict mediation work in the text domain. We formalize why it is necessary, showing that strict mediation makes representation quality empirically testable while history-leaky architectures break this connection. We then introduce textual latent states, which are discrete, interpretable, and variable-length, and factorized GRPO (fGRPO), a tree-structured reinforcement learning method that enforces strict mediation during training. Experiments on TextWorld and ScienceWorld show preserved one-step prediction accuracy alongside up to 57\% gains in representation quality and 98\% improvements in rollout performance, increasing with task complexity and horizon.
\end{abstract}

\section{Introduction}
World models learn the dynamics of an environment so that an agent can predict, plan, and reason without further interaction \citep{sutton91, ha2018world}. In partially observed settings, where observations reveal only fragments of the underlying state, a world model must maintain a latent representation that summarizes the interaction history. This latent state is the object on which rollouts operate, the branching point for counterfactual queries, and the compressed summary that replaces an ever-growing history. A large body of work has focused on building better dynamics models \citep{hafner2019dream, hafner2019planet, hafner2023mastering, schrittwieser2020mastering}, but a dynamics model is only as good as the representation it operates on. The central question is not only how to model dynamics, but how to know whether the learned representation is adequate.
This question is difficult to answer for an architectural reason. In many existing world model architectures, the latent state coexists with other information that can substitute for it. Sequence models condition on the full observation history and maintain no explicit latent state \citep{yang2024evaluating, li2025word}. Prompting-based approaches treat the LLM context as an implicit state without learning a compressed representation \citep{hao2023reasoning, wang2024can}. Architectures that introduce a latent state but allow the decoder to also access the history may ignore it entirely: the model can achieve strong predictive performance while the representation carries no useful information \citep{xiang2025pan, ge2024worldgpt}. We call this the \emph{identifiability problem}: when information can bypass the latent state, predictive accuracy does not constrain representation quality.
We propose strict latent state mediation to address this: once a latent state is generated, all subsequent transitions and predictions depend only on that state and the action. This mirrors the POMDP belief-state factorization \citep{kaelbling1998pomdp, murphy2000pomdp}; classical world models with continuous latent states enforce it by construction \citep{hafner2019dream, hafner2019planet}. However, recent LLM-based approaches frequently relax this constraint through large contexts or hybrid conditioning, enabling information to bypass the latent state. While this design improves predictive flexibility, it breaks the connection between predictive performance and representation quality. A key consequence, which motivates our approach, is that strict mediation makes representation quality \emph{empirically testable}. If a strictly mediated model predicts accurately, the latent state must contain all task-relevant information; if it predicts poorly, the representation is insufficient.
We formalize this principle and derive three consequences that motivate our approach. Under strict mediation, optimal predictive accuracy implies that the latent state is a sufficient statistic of history (Proposition~\ref{prop:sufficiency}). Under leaky architectures, optimal accuracy is compatible with an uninformative latent state (Proposition~\ref{prop:leaky}). Strict mediation also structurally aligns training and rollout conditions, so that the model never conditions on information during training that is unavailable during rollout, enabling bounded error propagation when dynamics are contractive, whereas history-based and leaky rollouts suffer geometric error growth (Proposition~\ref{prop:consistency}).
The central technical challenge is enforcing strict mediation in text-based world models, where observations and actions are natural language. Textual latent states align with the LLM backbone, enable inspection, and provide an adaptive information bottleneck \citep{hafner2019dream, hafner2019planet, hafner2025training}. However, they are discrete and non-differentiable, precluding variational training, and expressive decoders make ELBO objectives prone to posterior collapse \citep{bowman2016sentencevae, he2019deltavae}. We address this by treating latent state generation as a stochastic policy and introduce \textbf{factorized GRPO (fGRPO)}, a tree-structured extension that enforces strict mediation during training.
We evaluate on TextWorld \citep{cote2018textworld} and ScienceWorld \citep{wang2022scienceworld}. Our approach preserves one-step prediction accuracy (Proposition~\ref{prop:sufficiency}), yields higher state-level F1 under matched prediction (Proposition~\ref{prop:leaky}), and improves rollout stability, with gains growing from 5\% at horizon~1 to 80\% at horizon~9 (Proposition~\ref{prop:consistency}).

\section{Problem Setup: Text Based POMDPs}
\label{sec:setup}
\begin{figure}[t]
    \centering
    \includegraphics[width=0.9\linewidth]{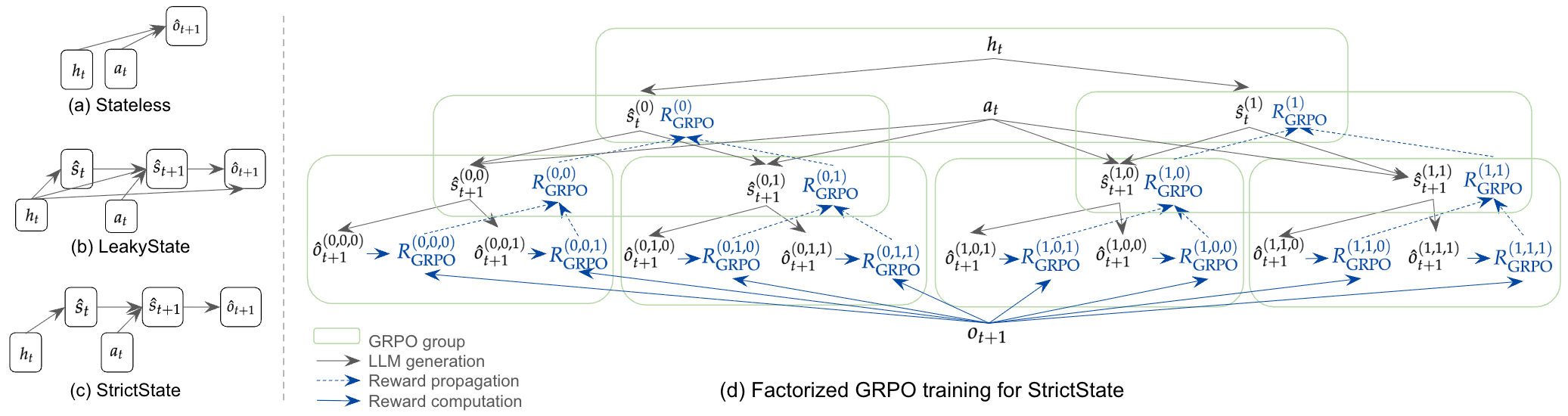}
    \caption{Three formulations of a world model in a partially observed setting: (a) Stateless, (b) LeakyState, and (c) StrictState. StrictState is proposed to be trained using (d) factorized GRPO.}
    \label{fig:combined}
    \vspace{-6pt}
    \end{figure}

We are given an offline dataset of trajectories $(o_0, a_0, o_1, a_1, \ldots, o_T)$ from a partially observable environment, where observations $o_t$ and actions $a_t$ are textual. The true state is unobserved, and a single observation is generally insufficient to identify it, requiring aggregation over time. Our goal is to learn a model that predicts future observations using a latent belief state summarizing history.
\subsection{Latent Belief States and Predictive Sufficiency}
The interaction history at time $t$ is $h_t = (o_0, a_0, o_1, a_1, \ldots, a_{t-1}, o_t)$. Because the environment is partially observed, accurate prediction typically requires multi-step information. A latent belief state $\hat{s}_t \sim p_\theta(\cdot \mid h_t)$ is a learned summary of $h_t$, represented as a text string produced by a language model. Its structure is learned and need not match predefined variables. We require \emph{predictive sufficiency}.
\begin{definition}[Predictive Sufficiency]
\label{def:sufficiency}
A latent state $\hat{s}_t$ is sufficient if
\begin{equation}
p(o_{t+1} \mid h_t, a_t) = p(o_{t+1} \mid \hat{s}_t, a_t) \quad \text{for all } h_t, a_t, \quad \text{i.e.,} \quad o_{t+1} \perp h_t \mid \hat{s}_t, a_t.
\end{equation}
\end{definition}
A sufficient latent state retains all information from the past relevant for predicting the \emph{immediate next} observation. This is weaker than full POMDP belief-state sufficiency, which requires sufficiency for all future observations \citep{kaelbling1998pomdp}; multi-step sufficiency follows when the transition model is also accurate, as the state at each step is re-derived via the learned dynamics. Whether this holds depends on both the objective and architectural constraints. Multiple sufficient statistics may exist for a given prediction problem; we do not require minimality. Strict mediation guarantees that the learned state is sufficient when prediction is optimal (Proposition~\ref{prop:sufficiency}), while our reward design (Section~\ref{sec:reward}) encourages compactness by penalizing redundant and unnecessary facts.
\subsection{Architectural Considerations for Observation Prediction}
\label{sec:arch}
We study three formulations of
$p(o_{t+1} \mid h_t, a_t)$,
which differ in what information is available at prediction time, as illustrated in Figure~\ref{fig:combined}.
\paragraph{Stateless prediction.}
The simplest approach predicts directly from history: $\hat{o}_{t+1} \sim p_\theta(\cdot \mid h_t, a_t)$. This can be accurate but provides no explicit, compact representation and scales poorly with horizon.
\paragraph{Leaky latent prediction.}
A latent state is introduced but prediction still accesses history: $\hat{s}_t \sim p_\theta(\cdot \mid h_t)$, $\hat{o}_{t+1} \sim p_\theta(\cdot \mid \hat{s}_t, h_t, a_t)$. This \emph{leaky} setup allows the model to bypass the latent state, so predictive accuracy does not ensure that $\hat{s}_t$ is informative.
\paragraph{Strict latent state mediation.}
Our formulation requires all prediction to pass through latent state:
\begin{align}
\hat{s}_t \sim p_\theta(\cdot \mid h_t), \qquad
\hat{s}_{t+1} \sim p_\theta(\cdot \mid \hat{s}_t, a_t), \qquad
\hat{o}_{t+1} \sim p_\theta(\cdot \mid \hat{s}_{t+1}, a_t). \nonumber
\end{align}
The predictive distribution can be written as ($s_t$ denotes the true (unobserved) environment state, distinct from the learned latent state $\hat{s}_t$.)
\begin{equation}
\label{eq:pomdp_marginal}
p(o_{t+1} \mid h_t, a_t)
=
\sum_{\hat{s}_{t+1}} p(o_{t+1} \mid \hat{s}_{t+1}, a_t)
\sum_{\hat{s}_t} p(\hat{s}_{t+1} \mid \hat{s}_t, a_t)\, p(\hat{s}_t \mid h_t).
\end{equation}
We include $a_t$ in the observation model since in text-based environments, the observation rendered at step $t+1$ typically depends on the action taken. Strict mediation mirrors this factorization, ensuring all predictive information flows through the latent state. Under this constraint, the latent state is sufficient exactly when it supports accurate prediction without access to history.
These formulations differ only in what information is available at prediction time; we next examine the implications.

\section{Why Strict Mediation Is Necessary}
\label{sec:theory}
The architectural formulations in Section~\ref{sec:arch} differ only in what information is available at prediction time, but this distinction has a fundamental consequence: whether predictive performance constrains the learned representation.
\subsection{Sufficiency from Prediction}
Under strict mediation, all predictive information must pass through the latent state. This makes predictive sufficiency a direct consequence of predictive accuracy.
\begin{proposition}[Sufficiency under strict mediation]
\label{prop:sufficiency}
In a strictly mediated model with $p_\theta(o_{t+1} \mid \hat{s}_t, a_t)$ and $\hat{s}_t \sim p_\theta(\cdot \mid h_t)$, if
$
p_\theta(o_{t+1} \mid h_t, a_t) = p(o_{t+1} \mid h_t, a_t) \quad \forall h_t, a_t,
$
then $\hat{s}_t$ is a sufficient statistic:
$
p(o_{t+1} \mid h_t, a_t) = p(o_{t+1} \mid \hat{s}_t, a_t).
$
\end{proposition}
\textbf{Proof sketch.}
Under strict mediation, the decoder cannot access $h_t$. If two histories map to the same $\hat{s}_t$ but require different predictions, the model cannot match both, contradicting optimality.
This result is structurally simple, but important: it shows that under strict mediation, predictive accuracy becomes a \emph{diagnostic} for representation quality. Formally, optimal prediction implies $I(o_{t+1}; h_t \mid \hat{s}_t, a_t) = 0$; a full information-theoretic treatment and quantitative relaxation for suboptimal prediction are provided in Appendix~\ref{app:proofs}.
\subsection{The Leaky Bypass Problem}
\label{sec:leaky}
This guarantee is specific to strict mediation. When the decoder can access the history, predictive accuracy places no constraint on the latent state.
\begin{proposition}[Non-identifiability under leaky architectures]
\label{prop:leaky}
For a leaky model $p_\theta(o_{t+1} \mid \hat{s}_t, a_t, h_t)$, there exist parameters achieving optimal prediction while $I(\hat{s}_t; h_t)=0$.
\end{proposition}
\textbf{Proof sketch.}
Set $\hat{s}_t$ to a constant and predict directly from $(h_t, a_t)$.
\textbf{Implication.}
Predictive accuracy does not guarantee that the latent state is informative when history bypass is allowed. Strict mediation removes the structural mismatch. Full proof in Appendix~\ref{app:proofs}.
\subsection{Rollout Consistency and Stability}
Strict mediation also affects multi-step rollout. The key distinction is whether the information available at training matches rollout conditions.
\begin{proposition}[Training--inference consistency]
\label{prop:consistency}
In strictly mediated models, rollout and training operate on the same inputs $(\hat{s}_t, a_t)$, so one-step errors match. In leaky models, reliance on $h_t$ during training leads to degraded rollout performance.
\end{proposition}
Beyond consistency, strict mediation changes how errors propagate. In stateless and leaky models, errors accumulate through an expanding history, leading to worst-case geometric growth. In contrast, strictly mediated models propagate error through a single latent state, enabling bounded error under contraction assumptions. The bounded-error conclusion requires $L_S < 1$, a property of the learned dynamics, not the architecture; strict mediation ensures only the structural form of the recurrence, not the contraction property itself. Table~\ref{tab:error_comparison} summarizes these regimes; full derivations are in Appendix~\ref{app:rollout}.

\section{Enforcing Strict Mediation in the Text Domain}
\label{sec:method}
The preceding analysis establishes why strict mediation matters. We now address the central challenge: enforcing it in the text domain, where latent states are discrete, non-differentiable, and must be trained without variational objectives.
\subsection{Textual Latent States}
\label{sec:why_text}
We represent latent states as text rather than continuous vectors. While classical world models use continuous latents trained with variational objectives such as the ELBO \citep{hafner2019dream, hafner2019planet}, this is best suited to low-dimensional settings. In text-based environments, observations are free-form, horizons are long, and relevant information mixes entities, relations, and attributes, making textual states a better match.
Textual latent states offer four advantages.(1) \textbf{Alignment with LLMs:} encoding, transition, and decoding are implemented as different conditioning modes of a single model, avoiding lossy modality conversions. (2) \textbf{Adaptive bottleneck:} capacity scales with content via variable-length representations, rather than a fixed-dimensional constraint. (3) \textbf{Interpretability:} states are directly inspectable and can be evaluated against ground-truth facts, enabling representation-centric evaluation. (4) \textbf{Modality match:} observations are already semantic text, so textual states preserve their structure.
The main tradeoff is that textual states are discrete and non-differentiable, preventing standard reparameterization-based training. In addition, expressive decoders trained with ELBO-style objectives can ignore the latent state (posterior collapse) \citep{bowman2016sentencevae, he2019deltavae}. Under strict mediation, this immediately harms prediction since there is no history bypass. Overcoming this training difficulty is the core technical problem this paper addresses. We therefore train latent state generation using reinforcement learning (Section~\ref{sec:fgrpo}).
\paragraph{State schema.}
Strict mediation does not require a predefined latent-state vocabulary. In the most general form, the latent state may be any textual object passed between encoding, transition, and decoding. In this work, we instantiate latent states as JSON predicate--argument lists because they are easy to parse, inspect, and compare. For the main experiments, we further restrict predicates to a simulator-aligned vocabulary. This restriction is not required by the model, but enables more comprehensive state-level evaluation (Section~\ref{sec:metrics}).
\subsection{Model Architecture}
\label{sec:architecture_impl}
We implement strict mediation with a single pretrained autoregressive language model with parameters $\theta$. The model is used in three prompt-defined stages with shared parameters:\vspace{-4pt}
\begin{equation*}
\begin{array}{ll@{\qquad}ll}
\text{Encoding:}~\hat{s}_t \sim p_\theta(\cdot \mid h_t), & \text{(3)} &
\text{Predicting:}~\hat{s}_{t+1} \sim p_\theta(\cdot \mid \hat{s}_t, a_t), & \text{(4)} \\[4pt]
\text{Decoding:}~\hat{o}_{t+1} \sim p_\theta(\cdot \mid \hat{s}_{t+1}, a_t). & \text{(5)} & &
\end{array}
\end{equation*}
\addtocounter{equation}{3}
\vspace{-1pt} After $\hat{s}_t$ is generated, the transition and decoding stages receive no access to $h_t$. Thus, the latent state is the only channel from history to future prediction, both during training and inference.
\subsection{Training via Factorized GRPO}
\label{sec:fgrpo}
Section~\ref{sec:architecture_impl} defines the information flow; we now turn to the key algorithmic challenge: training under strict mediation when the latent space is discrete text. Textual latent states are non-differentiable, which precludes gradient-based optimization through the latent space. We treat each generation stage (encoding, transition, and decoding) as a stochastic policy over text sequences and optimize for predictive accuracy using reinforcement learning.
Standard GRPO \citep{shao2024deepseekmath} is incompatible with strict mediation: it operates on a single autoregressive continuation of a fixed prompt, so the transition and decoding steps would access the original history through the shared prompt context. We introduce \textbf{factorized GRPO} (fGRPO, Figure~\ref{fig:combined}), which extends GRPO to tree-structured, multi-prompt rollouts. For a given history $h_t$, we sample $A$ candidate latent states, then $B$ candidate next states per encoding, then $C$ candidate observations per transition, forming a tree with $h_t$ as root and $\hat{o}_{t+1}^{(a,b,c)}$ as leaves. From a single tuple $(h_t, a_t, o_{t+1})$, this yields $1 + A + AB$ GRPO groups (Figure~1d).
Rewards are computed at the leaves by comparing $\hat{o}_{t+1}^{(a,b,c)}$ with $o_{t+1}$ and propagated upward by averaging over children so each parent takes the mean reward of its immediate descendants; advantages are computed group-relative within each level. All groups are aggregated into a single batch to update the shared LLM parameters.
\subsection{Reward Design}
\label{sec:reward}
The reward signal combines predictive accuracy with structural regularization of the latent states. We summarize the key components below; full details are provided in Appendix~\ref{app:reward}.
\textbf{Observation accuracy.} The primary reward encourages latent states that support accurate next-observation prediction under strict mediation. Given a predicted observation $\hat{o}_{t+1}$ and ground truth $o_{t+1}$, we define $R_{\text{obs}}$ as the fact-level F1 score between the extracted fact sets of the two observations. This directly measures whether the model predicts the correct underlying environment facts, independent of surface form.
\textbf{Format validity.} A binary indicator $R_{\mathrm{fmt}}(\hat{s}) \in \{0, 1\}$ checks whether the generated latent state satisfies required structural constraints. For a rollout branch, validity is enforced jointly across stages so all generated latent states must be valid; invalid states receive a penalty on the observation reward rather than an additive bonus, preventing the model from producing well-formed but semantically vacuous states.
\textbf{Auxiliary state rewards.} We introduce three auxiliary rewards for valid states: (i) a support reward $R_{\text{sup}}$ that encourages consistency between the predicted observation and the underlying state facts (both derived from text); (ii) a deduplication reward $R_{\text{dedup}}$ that penalizes semantically redundant facts in the latent state; and (iii) an ablation reward $R_{\text{abl}}$ that measures the fraction of facts in the latent state whose removal degrades prediction accuracy, encouraging compactness. The computation of these rewards does not rely on the access to hidden environment states.
\textbf{Combined reward.} The final reward is:
\begin{equation}
R = \begin{cases}
R_{\text{obs}} \cdot \gamma_{\text{pen}}, & \text{if } R_{\text{fmt}} = 0, \\
R_{\text{obs}} + \lambda_{\text{dedup}} R_{\text{dedup}} + \lambda_{\text{abl}} R_{\text{abl}} + \lambda_{\text{sup}} R_{\text{sup}}, & \text{if } R_{\text{fmt}} = 1,
\end{cases}
\end{equation}
where $\gamma_{\text{pen}} \in [0,1]$ is a penalty factor and $\lambda_{\text{dedup}}, \lambda_{\text{abl}}, \lambda_{\text{sup}} \geq 0$ control the auxiliary reward weights.

\section{Experimental Evaluation}
Our evaluation targets \emph{representation quality}, rather than downstream control or planning performance. Under strict mediation, representation properties become empirically identifiable: if a latent state is sufficient, compact, and stable, these properties should manifest directly in prediction behavior.
\subsection{Datasets}
\label{sec:datasets}
Our evaluation requires environments that are partially observed for the agent but fully instrumentable for the developer: the model sees only textual observations, while we access the complete Markov state for evaluation. We use two text-based benchmarks: \textsc{TextWorld} \citep{cote2018textworld} and \textsc{ScienceWorld} \citep{wang2022scienceworld}. Both define a fully specified symbolic state but render only a lossy natural-language observation at each step; a single observation is generally insufficient to identify the underlying state.
We generate offline datasets by executing a diverse set of behavior policies (expert, $\epsilon$-greedy, exploration-then-oracle, random walk, and failure modes) and logging trajectories of observations and actions. For \textsc{TextWorld}, we generate $2,726$ trajectories ($121,054$ steps) across $261$ game configurations varying in world size ($1$--$12$ rooms), recipe length ($1$--$5$ ingredients), and skill requirements, with $24$ random seeds per configuration. For \textsc{ScienceWorld}, we generate $600$ trajectories ($19,481$ steps) across $30$ task types. Average trajectory length is $45$ steps for \textsc{TextWorld} and $33$ for \textsc{ScienceWorld}. The main training set contains $2,287$ trajectories ($102,920$ steps) for TW and $540$ trajectories ($17,573$ steps) for SW, with permanent holdout sets of $439$ TW trajectories and $60$ SW trajectories reserved for final evaluation. All model training is strictly offline, without further environment interaction. Full details on data generation, complexity sweeps, noise injection, ground-truth logging, and holdout design are provided in Appendix~\ref{app:data}.

\subsection{Baselines}
\label{sec:baselines}
We compare against three baselines corresponding to the architectural formulations in Section~\ref{sec:arch}, isolating the effect of latent state mediation and representation choice. \texttt{Stateless} follows the history-based formulation, predicting the next observation directly from the full history and action via supervised fine-tuning. It serves as an upper bound on one-step predictive performance but does not learn an explicit representation. \texttt{LeakyState} follows the leaky formulation, introducing a textual latent state while allowing prediction to condition on both the state and the full history; it is trained using standard GRPO without enforcing strict mediation, so the model can rely on the history bypass as discussed in Proposition~\ref{prop:leaky}.
StrictState-Emb follows the strict formulation but replaces textual latent states with continuous embeddings, as a fixed number of soft tokens prepended to the decoder input, trained via a variational objective $\mathcal{L}_{\text{ELBO}} = \mathbb{E}_{q_\phi(s_t|h_t)}[\log p_\theta(o_{t+1} | s_t, a_t)] - \beta \, \mathrm{KL}(q_\phi(s_t | h_t) \| p_\theta(s_t | s_{t-1}, a_{t-1}))$, mirroring classical world models such as RSSM and Dreamer.

\subsection{Evaluation Metrics}
\label{sec:metrics}
Our evaluation targets representation quality rather than downstream control, using three metrics. \textbf{Observation-level F1} (\texttt{obs.F1}) measures next-observation prediction quality: we extract structured facts (predicates and arguments) from both the predicted and ground-truth observations and compute F1 over the extracted fact sets (Appendix~\ref{app:fact_extraction}), directly measuring whether the model predicts the correct environment facts rather than matching surface text.
\textbf{State-level F1} (\texttt{state.F1}) evaluates latent states by parsing them into structured predicates and computing F1 against the subset of ground-truth facts needed to render the next observation\footnote{Ground-truth states are hidden from the model and used only for final evaluation.}. This rendering-derived subset provides an exact sufficiency target without requiring irrelevant simulator variables (Appendix~\ref{app:data}).
We additionally report the \textbf{compression ratio} (CR), defined as the ratio between the cumulative history length and the latent state length, measuring how aggressively the model compresses history into its representation. Finally, to evaluate temporal consistency, we measure \texttt{state.F1} across \textbf{multi-step rollouts} at increasing prediction horizons $H$; degradation as $H$ increases reflects the accumulation of prediction errors and allows comparison of rollout stability across models.
We use greedy decoding  during evaluation and report $\pm 1$ standard error over test examples as error bars\footnote{
Although fGRPO trains a stochastic text policy, sampling is used for exploration and advantage estimation. Evaluation uses the deterministic greedy policy to remove sampling variance and isolate whether the learned state mediates prediction.
}.

\subsection{Results}
\label{sec:results}
Our experiments evaluate whether enforcing strict mediation via fGRPO preserves predictive accuracy (Prop~\ref{prop:sufficiency}), yields more informative latent states (Prop~\ref{prop:leaky}), and improves rollout stability (Prop~\ref{prop:consistency}).
\paragraph{One-step prediction and state quality.} Table~\ref{tab:vs_n_ref} reports performance across architectures at increasing compositional difficulty $M$ (we exclude trivial cases with $M < 2$). StrictState-Text matches obs.F1 of LeakyState and Stateless baselines, confirming that strict mediation does not sacrifice predictive accuracy (Proposition~\ref{prop:sufficiency}).
The critical difference is in representation quality: under matched obs.F1, StrictState-Text achieves higher state.F1, and the advantage is largest in high-compositional settings rather than uniformly determined by model scale (e.g., up to +30\% for 0.5B, +21\% for 4B and +25\% for 32B on TextWorld).
This validates Proposition~\ref{prop:leaky}: leaky models predict well while learning uninformative states, because the history bypass absorbs the predictive burden.  A qualitative comparison of learned latent states is provided in Appendix~\ref{app:qualitative}.
Figure~\ref{fig:cr} further shows that StrictState better calibrates compression to reference complexity ($\rho=0.62$ vs.\ $0.18$), while LeakyState often over-compresses, producing states that are too short when more predictive facts are needed.
StrictState-Text also matches StrictState-Emb on obs.F1, suggesting textual states provide a more flexible bottleneck than fixed-dimensional embeddings for discrete, compositional facts.
\begin{table}[t]
\centering
\scriptsize
\setlength{\tabcolsep}{3.5pt}
\renewcommand{\arraystretch}{0.9}
\begin{tabular}{@{}c l c c c c c c@{}}
\toprule
\multirow{2}{*}{base} & \multirow{2}{*}{method}
& \multicolumn{3}{c}{obs.F1}
& \multicolumn{3}{c}{state.F1} \\
\cmidrule(lr){3-5} \cmidrule(lr){6-8}
& & $M\geq2$ & $M\geq3$ & $M\geq4$
  & $M\geq2$ & $M\geq3$ & $M\geq4$ \\
\midrule
\multicolumn{8}{c}{\textbf{TextWorld}} \\
\midrule
\multirow{4}{*}{Qwen2.5-0.5B}
& no-state & $0.529^{{\pm 0.005}}$ & $0.427^{{\pm 0.007}}$ & $0.413^{{\pm 0.009}}$ & - & - & - \\
& emb-state-strict & $0.501^{{\pm 0.006}}$ & $0.419^{{\pm 0.008}}$ & $0.391^{{\pm 0.008}}$ & - & - & - \\
& text-state-leaky & $0.520^{{\pm 0.005}}$ & $0.429^{{\pm 0.005}}$ & $0.402^{{\pm 0.004}}$ & $0.265^{{\pm 0.005}}$ & $0.191^{{\pm 0.005}}$ & $0.170^{{\pm 0.006}}$ \\
& text-state-strict  & $0.525^{{\pm 0.006}}$ & $0.434^{{\pm 0.007}}$ & $0.397^{{\pm 0.007}}$
        & $\mathbf{0.290}^{{\pm 0.004}}_{(+9\%)}$
        & $\mathbf{0.236}^{{\pm 0.005}}_{(+24\%)}$
        & $\mathbf{0.221}^{{\pm 0.006}}_{(+30\%)}$ \\
\midrule
\multirow{4}{*}{Qwen3-4B}
& no-state & $0.685^{{\pm 0.006}}$ & $0.642^{{\pm 0.008}}$ & $0.634^{{\pm 0.008}}$ & - & - & - \\
& emb-state-strict & $0.609^{{\pm 0.007}}$ & $0.558^{{\pm 0.008}}$ & $0.503^{{\pm 0.009}}$ & - & - & - \\
& text-state-leaky & $0.665^{{\pm 0.006}}$ & $0.639^{{\pm 0.007}}$ & $0.633^{{\pm 0.005}}$ & $0.294^{{\pm 0.006}}$ & $0.281^{{\pm 0.006}}$ & $0.271^{{\pm 0.007}}$ \\
& text-state-strict  & $0.680^{{\pm 0.006}}$ & $0.631^{{\pm 0.007}}$ & $0.614^{{\pm 0.008}}$
        & $\mathbf{0.310}^{{\pm 0.007}}_{(+5\%)}$
        & $\mathbf{0.326}^{{\pm 0.007}}_{(+16\%)}$
        & $\mathbf{0.329}^{{\pm 0.008}}_{(+21\%)}$ \\
\midrule
 \multirow{4}{*}{Qwen3-32B}
 & no-state & $0.762^{\pm 0.005}$ & $0.673^{\pm 0.008}$ & $0.655^{\pm 0.009}$ & - & - & - \\
 & emb-state-strict & $0.641^{\pm 0.008}$ & $0.612^{\pm 0.008}$ & $0.547^{\pm 0.009}$ & - & - & - \\
 & text-state-leaky
 & $0.723^{\pm 0.005}$ & $0.698^{\pm 0.008}$ & $0.693^{\pm 0.009}$ & $0.392^{\pm 0.005}$ & $0.419^{\pm 0.007}$ & $0.425^{\pm 0.006}$ \\
 & text-state-strict
 & $0.752^{\pm 0.006}$ & $0.710^{\pm 0.007}$ & $0.706^{\pm 0.011}$
 & $\mathbf{0.411}^{\pm 0.005}_{(+5\%)}$
 & $\mathbf{0.494}^{\pm 0.006}_{(+18\%)}$
 & $\mathbf{0.531}^{\pm 0.007}_{(+25\%)}$ \\
\midrule
\multicolumn{8}{c}{\textbf{ScienceWorld}} \\
\midrule
& & $M\geq5$ & $M\geq10$ & $M\geq30$
  & $M\geq5$ & $M\geq10$ & $M\geq30$ \\
\midrule
\multirow{4}{*}{Qwen2.5-0.5B}
 & no-state & {$0.520^{\pm 0.014}$} & {$0.297^{\pm 0.016}$} & {$0.197^{\pm 0.019}$} & - & - & - \\
 & emb-state-strict & {$0.408^{\pm 0.015}$} & {$0.265^{\pm 0.017}$} & {$0.130^{\pm 0.018}$} & - & - & - \\
 & text-state-leaky
 & {$0.517^{\pm 0.014}$} & {$0.291^{\pm 0.016}$} &{$0.167^{\pm 0.018}$} & {$0.298^{\pm 0.015}$} & {$0.203^{\pm 0.018}$} & {$0.070^{\pm 0.020}$} \\
 & text-state-strict
 & {$0.512^{\pm 0.013}$}
 & {$0.280^{\pm 0.015}$}
 & {$0.200^{\pm 0.019}$}
 & {$\mathbf{0.323}^{\pm 0.014}_{(+8\%)}$}
 & {$\mathbf{0.254}^{\pm 0.017}_{(+25\%)}$}
 & {$\mathbf{0.102}^{\pm 0.021}_{(+46\%)}$} \\
\midrule
\multirow{4}{*}{Qwen3-4B}
& no-state & $0.789^{\pm 0.017}$ & $0.765^{\pm 0.019}$ & $0.562^{\pm 0.022}$ & - & - & - \\
& emb-state-strict & {$0.651^{\pm 0.018}$} & {$0.631^{\pm 0.020}$} & {$0.309^{\pm 0.023}$} & - & - & - \\
& text-state-leaky & $0.722^{\pm 0.016}$ & $0.730^{\pm 0.018}$ & $0.356^{\pm 0.022}$ & $0.533^{\pm 0.015}$ & $0.560^{\pm 0.018}$ & $0.365^{\pm 0.024}$ \\
& text-state-strict  & $0.774^{\pm 0.016}$
        & $0.793^{\pm 0.018}$
        & $0.728^{\pm 0.021}$
        & $\mathbf{0.573}^{\pm 0.014}_{(+8\%)}$
        & $\mathbf{0.620}^{\pm 0.017}_{(+11\%)}$
        & $\mathbf{0.574}^{\pm 0.027}_{(+57\%)}$ \\
\midrule
 \multirow{4}{*}{Qwen3-32B}
 & no-state & $0.783^{\pm 0.018}$ & $0.778^{\pm 0.020}$ & $0.565^{\pm 0.023}$ & - & - & - \\
 & emb-state-strict & $0.681^{\pm 0.019}$ & $0.685^{\pm 0.021}$ & $0.487^{\pm 0.024}$ & - & - & - \\
 & text-state-leaky
 & $0.770^{\pm 0.017}$ & $0.784^{\pm 0.019}$ & $0.561^{\pm 0.022}$ & $0.583^{\pm 0.016}$ & $0.610^{\pm 0.019}$ & $0.415^{\pm 0.025}$ \\
 & text-state-strict
 & $0.781^{\pm 0.016}$
 & $0.795^{\pm 0.018}$
 & $0.775^{\pm 0.022}$
 & $\mathbf{0.623}^{\pm 0.015}_{(+7\%)}$
 & $\mathbf{0.670}^{\pm 0.018}_{(+10\%)}$
 & $\mathbf{0.624}^{\pm 0.024}_{(+50\%)}$ \\
\bottomrule
\end{tabular}
\vspace{6pt}
\caption{\textbf{Performance across base models and compositional difficulty.}
$M$ counts facts needed to render an observation; larger $M$ means higher complexity. Thresholds are chosen separately by within-environment quantiles to reflect each benchmark's scale.
}
\label{tab:vs_n_ref}
\vspace{-3mm}
\end{table}

\paragraph{Rollout stability.}
Table~\ref{tab:vs_future} reports performance across increasing prediction horizons $H$. StrictState-Text consistently outperforms LeakyState, with gains growing from +5\% at $H{=}1$ to +98\% at $H{=}9$ on TextWorld (4B model). This supports Proposition~\ref{prop:consistency}: leaky models suffer from training-inference mismatch, relying on history during training that is unavailable during rollout, while strict models operate under identical conditions in both regimes. Consistent with the error propagation analysis in Appendix~\ref{app:rollout}, the gap widens with horizon, confirming that strict mediation mitigates cumulative error propagation. The long-horizon gains further suggest that the compactness rewards (e.g., $R_\text{abl}$) do not simply collapse the textual state to immediately visible facts; future-relevant facts can still be retained when they support the learned rollout dynamics.
\begin{table}[t]
\centering
 \scriptsize
\setlength{\tabcolsep}{3.5pt}
\renewcommand{\arraystretch}{0.95}
\begin{tabular}{@{}c l c c c c c c@{}}
\toprule
\multirow{2}{*}{base} & \multirow{2}{*}{method}
& \multicolumn{3}{c}{obs.F1}
& \multicolumn{3}{c}{state.F1} \\
\cmidrule(lr){3-5} \cmidrule(lr){6-8}
 &  & $H{=}1$ & $H{=}5$ & $H{=}9$
 & $H{=}1$ & $H{=}5$ & $H{=}9$ \\
\midrule
\multicolumn{8}{c}{\textbf{TextWorld}} \\
\midrule
\multirow{4}{*}{Qwen2.5-0.5B}
 & no-state & $0.632^{\pm 0.029}$ & $0.511^{\pm 0.030}$ & $0.397^{\pm 0.031}$ & - & - & - \\
 & emb-state-strict & $0.601^{\pm 0.030}$ & $0.503^{\pm 0.029}$ & $0.401^{\pm 0.031}$ & - & - & - \\
 & text-state-leaky
 & $0.621^{\pm 0.028}$ & $0.504^{\pm 0.027}$ & $0.383^{\pm 0.030}$
 & ${0.255}^{\pm 0.016}$ & $0.176^{\pm 0.018}$ & $0.094^{\pm 0.020}$ \\
 & text-state-strict
 & $0.619^{\pm 0.029}$ & $0.524^{\pm 0.027}$ & $0.401^{\pm 0.028}$
 & $\mathbf{0.291}^{\pm 0.014}_{(+14\%)}$
 & $\mathbf{0.252}^{\pm 0.015}_{(+43\%)}$
 & $\mathbf{0.111}^{\pm 0.012}_{(+18\%)}$ \\
\midrule
\multirow{4}{*}{Qwen3-4B}
 & no-state & $0.728^{\pm 0.027}$ & $0.600^{\pm 0.030}$ & $0.505^{\pm 0.033}$ & - & - & - \\
 & emb-state-strict & $0.678^{\pm 0.032}$ & $0.512^{\pm 0.033}$ & $0.402^{\pm 0.034}$ & - & - & - \\
 & text-state-leaky
 & $0.693^{\pm 0.027}$ & $0.549^{\pm 0.029}$ & $0.433^{\pm 0.030}$
 & ${0.290}^{\pm 0.019}$ & $0.201^{\pm 0.019}$ & $0.106^{\pm 0.022}$ \\
 & text-state-strict
 & $0.729^{\pm 0.025}$ & $0.608^{\pm 0.026}$ & $0.570^{\pm 0.031}$
 & $\mathbf{0.305}^{\pm 0.012}_{(+5\%)}$
 & $\mathbf{0.277}^{\pm 0.014}_{(+38\%)}$
 & $\mathbf{0.210}^{\pm 0.014}_{(+98\%)}$ \\
\midrule
\multirow{4}{*}{Qwen3-32B}
 & no-state & $0.731^{\pm 0.026}$ & $0.617^{\pm 0.028}$ & $0.543^{\pm 0.030}$ & - & - & - \\
 & emb-state-strict & $0.703^{\pm 0.029}$ & $0.538^{\pm 0.031}$ & $0.419^{\pm 0.032}$ & - & - & - \\
 & text-state-leaky
 & $0.718^{\pm 0.026}$ & $0.583^{\pm 0.028}$ & $0.497^{\pm 0.029}$
 & $0.312^{\pm 0.018}$ & $0.243^{\pm 0.019}$ & $0.142^{\pm 0.021}$ \\
 & text-state-strict
 & $0.726^{\pm 0.024}$ & $0.659^{\pm 0.025}$ & $0.601^{\pm 0.029}$
 & $\mathbf{0.331}^{\pm 0.013}_{(+6\%)}$
 & $\mathbf{0.308}^{\pm 0.014}_{(+27\%)}$
 & $\mathbf{0.247}^{\pm 0.013}_{(+74\%)}$ \\
\midrule
\multicolumn{8}{c}{\textbf{ScienceWorld}} \\
\midrule
\multirow{4}{*}{Qwen2.5-0.5B}
 & no-state & {$0.565^{\pm 0.033}$} & {$0.341^{\pm 0.035}$} & {$0.201^{\pm 0.036}$} & - & - & - \\
 & emb-state-strict & {$0.438^{\pm 0.034}$} & {$0.290^{\pm 0.035}$} & {$0.203^{\pm 0.036}$} & - & - & - \\
 & text-state-leaky & {$0.560^{\pm 0.032}$} & {$0.351^{\pm 0.034}$} & {$0.223^{\pm 0.033}$}
 & {$0.301^{\pm 0.034}$} & {$0.211^{\pm 0.033}$} & {$0.104^{\pm 0.035}$} \\
 & text-state-strict  & {$0.556^{\pm 0.035}$} & {$0.361^{\pm 0.036}$} & {$0.250^{\pm 0.037}$}
 & {$\mathbf{0.331}^{\pm 0.033}_{(+10\%)}$}
 & {$\mathbf{0.250}^{\pm 0.035}_{(+18\%)}$}
 & {$\mathbf{0.145}^{\pm 0.036}_{(+39\%)}$} \\
\midrule
\multirow{4}{*}{Qwen3-4B}
 & no-state & $0.793^{{\pm 0.032}}$ & $0.728^{{\pm 0.033}}$ & $0.654^{{\pm 0.033}}$ & - & - & - \\
 & emb-state-strict & {$0.615^{\pm 0.032}$} & {$0.608^{\pm 0.033}$} & {$0.597^{\pm 0.034}$} & - & - & - \\
 & text-state-leaky & $0.695^{{\pm 0.031}}$ & $0.636^{{\pm 0.034}}$ & $0.678^{{\pm 0.035}}$
 & $0.502^{{\pm 0.036}}$ & $0.251^{{\pm 0.038}}$ & $0.218^{{\pm 0.030}}$ \\
 & text-state-strict  & $0.788^{{\pm 0.032}}$ & $0.792^{{\pm 0.033}}$ & $0.775^{{\pm 0.033}}$
 & $\mathbf{0.536}^{{\pm 0.036}}_{(+7\%)}$
 & $\mathbf{0.482}^{{\pm 0.039}}_{(+92\%)}$
 & $\mathbf{0.418}^{{\pm 0.030}}_{(+91\%)}$\\
\midrule
 \multirow{4}{*}{Qwen3-32B}
 & no-state & $0.823^{\pm 0.031}$ & $0.776^{\pm 0.033}$ & $0.703^{\pm 0.034}$ & - & - & - \\
 & emb-state-strict & $0.702^{\pm 0.032}$ & $0.687^{\pm 0.035}$ & $0.671^{\pm 0.034}$ & - & - & - \\
 & text-state-leaky
 & $0.728^{\pm 0.031}$ & $0.671^{\pm 0.032}$ & $0.694^{\pm 0.035}$
 & $0.557^{\pm 0.035}$ & $0.298^{\pm 0.037}$ & $0.263^{\pm 0.030}$ \\
 & text-state-strict
 & $0.822^{\pm 0.031}$ & $0.819^{\pm 0.033}$ & $0.798^{\pm 0.033}$
 & $\mathbf{0.598}^{\pm 0.034}_{(+7\%)}$
 & $\mathbf{0.547}^{\pm 0.038}_{(+84\%)}$
 & $\mathbf{0.501}^{\pm 0.030}_{(+90\%)}$ \\
\bottomrule
\end{tabular}
\vspace{6pt}
\caption{\textbf{Performance across base models and prediction horizons.}
We report observation and state prediction performance at increasing horizons $H$.}
\label{tab:vs_future}
\vspace{-3mm}
\end{table}

\paragraph{Reward ablation.}
Table~\ref{tab:ablation_rewards} evaluates the contribution of each reward component. At horizon $H{=}1$, obs.F1 is comparable across configurations, but differences emerge at longer horizons, revealing that the auxiliary rewards improve the state's capacity to support multi-step rollout. Observation accuracy alone ($R_{\text{obs}}$) produces largely invalid latent states; adding format validity ($R_{\text{fmt}}$) enforces structure but reduces obs.F1, as the model must now route prediction through a structured bottleneck. The support reward ($R_{\text{sup}}$) is critical: it recovers short-horizon obs.F1 and produces the largest long-horizon gains ($H{=}9$: $0.345 \rightarrow 0.541$), indicating that aligning the latent state with predicted observations directly improves rollout stability. Deduplication ($R_{\text{ded}}$) and ablation ($R_{\text{abl}}$) further increase compactness (CR: $24.3 \rightarrow 58.2$) while maintaining long-horizon performance, yielding states that are both minimal and predictively sufficient over extended rollouts.

\vspace{-2mm}
\section{Related Work}
\label{sec:related}
\vspace{-2mm}
\begin{figure}[t]
\centering
\begin{minipage}[b]{0.47\textwidth}
\centering
\scriptsize
\setlength{\tabcolsep}{2.5pt}
\begin{tabular}{lccc c}
\toprule
\textbf{Ablation} & \multicolumn{3}{c}{\textbf{Obs.F1}} & \textbf{CR} \\
\cmidrule(lr){2-4}
 & $H{=}1$ & $H{=}5$ & $H{=}9$ & \\
\midrule
$R_{\text{obs}}$ only
& $0.733^{\pm .031}$
& $0.494^{\pm .029}$
& $0.408^{\pm .032}$
& -- \\
\ $+R_{\text{fmt}}$
& $0.685^{\pm .027}$
& $0.468^{\pm .026}$
& $0.345^{\pm .027}$
& 47.8 \\
\ \ $+R_{\text{sup}}$
& $0.730^{\pm .024}$
& $0.597^{\pm .029}$
& $0.541^{\pm .031}$
& 24.3 \\
\ \ \ $+R_{\text{ded}}$
& $0.729^{\pm .025}$
& $0.607^{\pm .030}$
& $0.562^{\pm .033}$
& 46.2 \\
\ \ \ \ $+R_{\text{abl}}$
& $0.729^{\pm .025}$
& $0.608^{\pm .026}$
& $0.570^{\pm .031}$
& 58.2 \\
\bottomrule
\end{tabular}
\vspace{2pt}
\captionof{table}{Ablation over reward components on TextWorld with Qwen3-4B. Each row adds one reward to previous configs. Full reward in Appendix~\ref{app:reward}.}
\label{tab:ablation_rewards}
\end{minipage}
\hfill
\begin{minipage}[b]{0.5\textwidth}
\centering
\includegraphics[width=1\textwidth,height=0.65\textwidth,keepaspectratio]{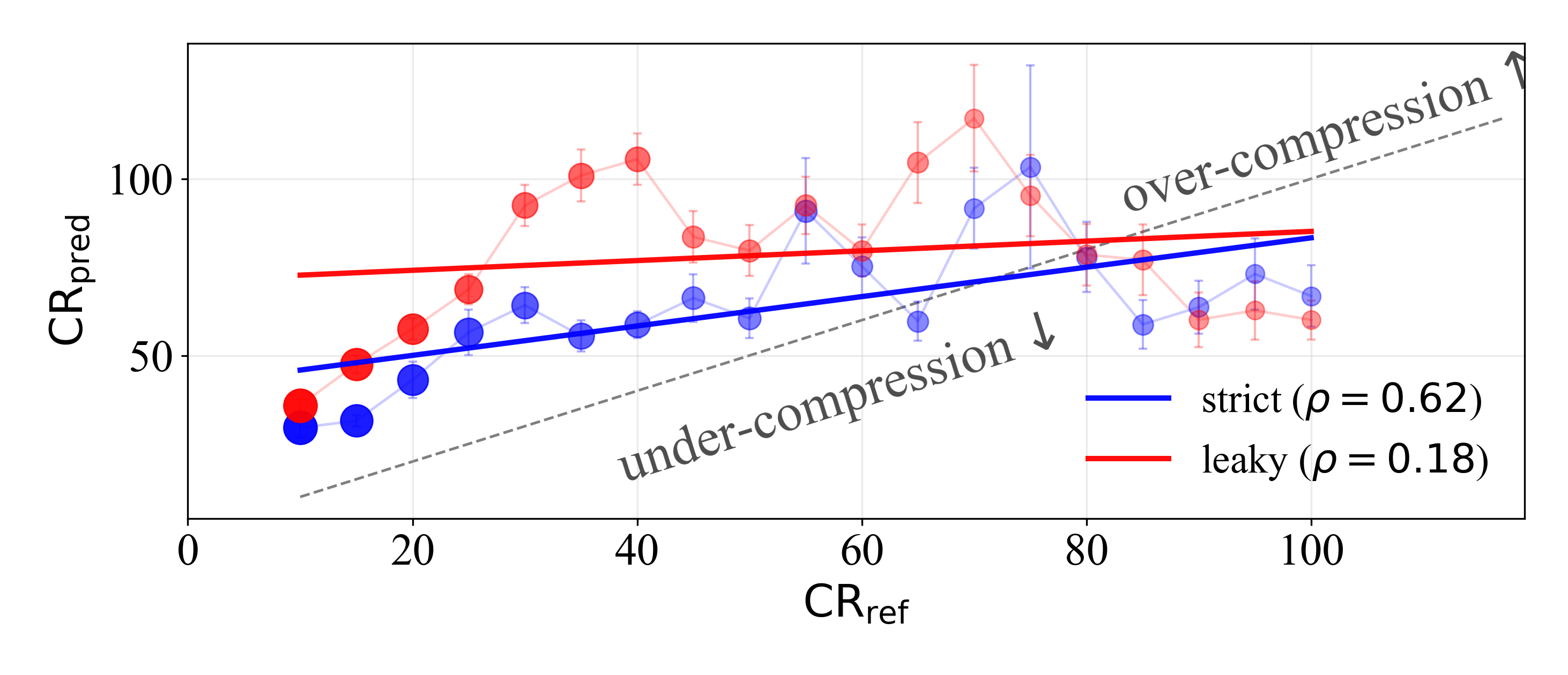}
\captionof{figure}{Compression ratio. $\mathrm{CR} = \tfrac{\text{history words}}{\text{state words}}$, with $\mathrm{CR}_{\mathrm{ref}}$ by ground truth, $\mathrm{CR}_{\mathrm{pred}}$ by predictions. Points are binned means; solid lines are linear fits (legend: Pearson $\rho$). Dashed line: $\mathrm{CR}_{\mathrm{pred}}{=}\mathrm{CR}_{\mathrm{ref}}$.}
\label{fig:cr}
\end{minipage}
\end{figure}

\textbf{Latent world models.} Classical world models such as RSSM \citep{hafner2019planet}, Dreamer \cite{hafner2019dream, hafner2020mastering, hafner2025training}, and MuZero \citep{schrittwieser2020mastering} learn continuous latent states with explicit transition models, enforcing a Markov factorization in which predictions depend only on the current state and action. Our work shares this commitment to strict mediation but operates in the text domain, where latent states are discrete and variational training is poorly matched (Section~\ref{sec:why_text}). More recent approaches relax this factorization: DreamerV4 \citep{hafner2025training} conditions on the full sequence of past latent embeddings, and PAN \citep{xiang2025pan} uses an autoregressive LLM backbone over the entire latent history. Our analysis in Section~\ref{sec:leaky} formalizes the consequence: when the architecture permits history bypass, predictive accuracy does not constrain the representation.
\textbf{LLMs as world models.} RAP \citep{hao2023reasoning} formulates reasoning as planning over textual states using a frozen LLM, enforcing strict mediation by construction but operating in fully observed MDPs without learning. \citet{wang2024can} evaluate LLMs on single-step state prediction given the full ground-truth state. \citet{li2025word} fine-tune LLMs on offline trajectories but condition on the full observation history without learning a compressed latent state. In all cases, either partial observability is absent, or no explicit latent representation is learned, or both. Our work addresses both. Since even a fine-tuned history-conditioned model fails to produce identifiable representations, the limitation is structural rather than a matter of model capacity or prompting strategy.
\textbf{Representation learning for sequential decision-making.} The information bottleneck principle \citep{tishby2000} provides the theoretical foundation for compressed representations that retain task-relevant information. JEPA-style methods \citep{lecun2022path} learn representations via embedded prediction rather than reconstruction. Our approach retains this signal but replaces variational training with reinforcement learning and enforces strict mediation to ensure identifiable sufficiency. Detailed comparison with prior work is in Table~\ref{tab:relatedwork} in appendix.

\vspace{-2mm}
\section{Conclusion}
\label{sec:conc}
We have shown that strict latent state mediation can be effectively enforced in text-based world models through factorized GRPO, yielding identifiable, interpretable representations whose quality is directly diagnosable from predictive performance. Confirmed empirically across compositional difficulty levels and prediction horizons, the resulting textual latent states compose entities, relations, and attributes as variable-length text, opening the door to LLM-grounded reward signals in settings where precise reward models are unavailable. The natural next steps are extending strict mediation to recurrent or hierarchical encoders that process history incrementally, and demonstrating that the representation-level gains established here translate into improved planning and task completion.
\textbf{Limitations.} Our evaluation focuses on representation quality rather than downstream control or planning, and is limited to two synthetic, instrumented environments (TextWorld and ScienceWorld), which may not fully reflect real-world complexity. Two structural limitations also merit discussion. First, strict mediation places the full burden of history compression on a single-shot encoder with no incremental refinement, so any missed information is irrecoverable, since removing history bypass also removes the decoder's ability to compensate for encoder errors. Second, while strict mediation and fGRPO are environment-agnostic, the auxiliary rewards that shape representations into interpretable, compact forms rely on structured fact extraction; extending these to less structured domains will require alternatives like learned similarity metrics or LLM-based evaluation.
\textbf{Broader Impacts}. This work is primarily methodological and may improve the reliability and interpretability of AI systems. Potential risks include misuse of more capable generative models, though our study is limited to offline benchmarks and does not involve deployment.

\bibliographystyle{unsrtnat}
\bibliography{reference}
\appendix
\newpage
\section{Proofs and Information-Theoretic Analysis}
\label{app:proofs}
\subsection{Full Proof of Proposition~\ref{prop:sufficiency}}
\begin{proof}
Under strict mediation, the model's predictive distribution for a given history $h_t$ and action $a_t$ is obtained by marginalizing over the encoder:
\begin{equation}
p_\theta(o_{t+1} \mid h_t, a_t) = \sum_{\hat{s}_t} p_\theta(o_{t+1} \mid \hat{s}_t, a_t) \, p_\theta(\hat{s}_t \mid h_t),
\end{equation}
where $p_\theta(o_{t+1} \mid \hat{s}_t, a_t, h_t) = p_\theta(o_{t+1} \mid \hat{s}_t, a_t)$ by the strict mediation constraint.
Suppose the model achieves optimal prediction: $p_\theta(o_{t+1} \mid h_t, a_t) = p(o_{t+1} \mid h_t, a_t)$ for all $h_t, a_t$. We show that the latent state is a sufficient statistic, i.e., $p(o_{t+1} \mid h_t, a_t) = p(o_{t+1} \mid \hat{s}_t, a_t)$ almost surely under the encoder.
For any latent state $\hat{s}_t$ in the support of $p_\theta(\cdot \mid h_t)$, define the set of histories that can produce it: $\mathcal{H}(\hat{s}_t) = \{h_t : p_\theta(\hat{s}_t \mid h_t) > 0\}$. For each $h_t \in \mathcal{H}(\hat{s}_t)$, the optimal prediction condition gives:
\begin{equation}
p(o_{t+1} \mid h_t, a_t) = \sum_{\hat{s}} p_\theta(o_{t+1} \mid \hat{s}, a_t) \, p_\theta(\hat{s} \mid h_t).
\end{equation}
Now suppose, for contradiction, that there exist $h_t, h_t' \in \mathcal{H}(\hat{s}_t)$ with $p(o_{t+1} \mid h_t, a_t) \neq p(o_{t+1} \mid h_t', a_t)$. Consider the difference:
\begin{equation}
p(o_{t+1} \mid h_t, a_t) - p(o_{t+1} \mid h_t', a_t) = \sum_{\hat{s}} p_\theta(o_{t+1} \mid \hat{s}, a_t) \left[ p_\theta(\hat{s} \mid h_t) - p_\theta(\hat{s} \mid h_t') \right].
\end{equation}
The right-hand side is a linear combination of the decoder distributions $p_\theta(o_{t+1} \mid \hat{s}, a_t)$, weighted by differences in encoder probabilities. For this to hold for \emph{all} observation values $o_{t+1}$ simultaneously, the encoder must assign different mixing weights to latent states with different decoder distributions. That is, the encoder must separate histories with different predictive distributions into different regions of the latent space.
This means that for any $\hat{s}_t$ in the support, all histories $h_t$ satisfying $p_\theta(\hat{s}_t \mid h_t) > 0$ and contributing non-negligibly to the mixture at $\hat{s}_t$ must share the same true predictive distribution. Formally, for encoder-decoder pairs that achieve optimal prediction, $p(o_{t+1} \mid h_t, a_t)$ must be constant across histories in the effective support of $p_\theta(\cdot \mid h_t)$ that maps to any given $\hat{s}_t$, yielding:
\begin{equation}
p(o_{t+1} \mid h_t, a_t) = p_\theta(o_{t+1} \mid \hat{s}_t, a_t) \quad \text{a.s. under } p_\theta(\hat{s}_t \mid h_t).
\end{equation}
Therefore, $o_{t+1} \perp h_t \mid \hat{s}_t, a_t$ almost surely under the joint distribution induced by the encoder, establishing predictive sufficiency. \qedhere
\end{proof}
\subsection{Full Proof of Proposition~\ref{prop:leaky}}
\begin{proof}
We construct parameters $\theta^*$ satisfying both claimed properties. Define the encoder to be a constant function: $p_{\theta^*}(\hat{s}_t \mid h_t) = \delta_c$ for all $h_t$, where $c$ is a fixed token string and $\delta_c$ is a point mass. Define the decoder to ignore $\hat{s}_t$ and predict from the history:
\begin{equation}
p_{\theta^*}(o_{t+1} \mid \hat{s}_t, h_t, a_t) = p(o_{t+1} \mid h_t, a_t).
\end{equation}
This is achievable with sufficient model capacity, since the decoder has access to $(h_t, a_t)$.
Property 1: By construction, $p_{\theta^*}(o_{t+1} \mid \hat{s}_t, h_t, a_t) = p(o_{t+1} \mid h_t, a_t)$ for all $h_t, a_t$, so the model achieves optimal prediction.
Property 2: Since $\hat{s}_t = c$ for all $h_t$, the latent state is a constant independent of the history. Therefore $I(\hat{s}_t; h_t) = H(\hat{s}_t) - H(\hat{s}_t \mid h_t) = 0 - 0 = 0$. \qedhere
\end{proof}
\paragraph{Remark.} It is instructive to verify that the constant-encoder construction fails under strict mediation. If $\hat{s}_t = c$ for all $h_t$, the transition produces $\hat{s}_{t+1} \sim p_\theta(\cdot \mid c, a_t)$, which depends only on the action. The decoder receives $\hat{s}_{t+1}$ with no access to $h_t$, so predictions become history-independent: $p_\theta(o_{t+1} \mid h_t, a_t) = p_\theta(o_{t+1} \mid a_t)$. In any environment where history carries predictive information beyond the action, this model cannot achieve optimal prediction. The constant-encoder construction is ruled out by the architecture itself.
\subsection{Full Proof of Proposition~\ref{prop:consistency}}
\begin{proof}
Under strict mediation, the model's prediction at each step depends only on $(\hat{s}_t, a_t)$. During training, the model receives $(\hat{s}_t, a_t)$ and produces $p_\theta(o_{t+1} \mid \hat{s}_t, a_t)$. During one-step rollout, the model receives the same input types and produces the same function. The information available to the model is identical in both regimes, so the one-step prediction error is the same.
In multi-step rollout, the latent state $\hat{s}_{t+k}$ is generated recursively by the learned transition function rather than encoded from the true history. The marginal distribution over latent states therefore shifts relative to training (covariate shift). However, under strict mediation, this shift is governed entirely by the transition function's error propagation, formalized in Appendix~\ref{app:rollout} via the contraction analysis: when $L_S < 1$, the state error remains bounded. Crucially, the \emph{form} of the input does not change between training and rollout; the model always conditions on $(\hat{s}_t, a_t)$ and never on $h_t$. Therefore there is no structural mismatch of the kind that affects leaky architectures.
Under a leaky architecture, the model's training-time prediction is $p_\theta(o_{t+1} \mid \hat{s}_t, h_t, a_t)$, which may depend on $h_t$. During rollout, the true history is unavailable and must be replaced by the synthetic history of past predictions $(\hat{o}_{t+1}, \ldots, \hat{o}_{t+k})$. This introduces both covariate shift (the synthetic history has different statistics than the true history) and a structural input mismatch (the model was trained with access to true history but must operate without it). Any learned dependence on $h_t$ leads to degraded rollout performance that cannot be bounded by transition error alone. \qedhere
\end{proof}
\subsection{Information-Theoretic Restatement of Proposition~\ref{prop:sufficiency}}
\label{app:info_theory}
Proposition~\ref{prop:sufficiency} can be restated in information-theoretic terms. Under strict mediation, the encoding and decoding steps form a Markov chain:
\begin{equation}
h_t \to \hat{s}_t \to o_{t+1} \quad (\text{given } a_t),
\end{equation}
meaning that $p_\theta(o_{t+1} \mid \hat{s}_t, a_t, h_t) = p_\theta(o_{t+1} \mid \hat{s}_t, a_t)$. By the data processing inequality applied to this chain,
\begin{equation}
\label{eq:dpi}
I(h_t; o_{t+1} \mid a_t) \geq I(\hat{s}_t; o_{t+1} \mid a_t),
\end{equation}
since processing $h_t$ through the encoder into $\hat{s}_t$ cannot increase information about $o_{t+1}$. Equality holds if and only if $\hat{s}_t$ is a sufficient statistic of $h_t$ for $o_{t+1}$ given $a_t$.
Define the \emph{predictive information gap}:
\begin{equation}
\label{eq:info_gap}
\Delta_I \;=\; I(o_{t+1}; h_t \mid \hat{s}_t, a_t).
\end{equation}
This quantity measures how much information $h_t$ carries about $o_{t+1}$ beyond what is captured in $\hat{s}_t$. By the chain rule of mutual information:
\begin{equation}
I(o_{t+1}; h_t \mid a_t) = I(o_{t+1}; \hat{s}_t \mid a_t) + I(o_{t+1}; h_t \mid \hat{s}_t, a_t),
\end{equation}
so $\Delta_I = I(o_{t+1}; h_t \mid a_t) - I(o_{t+1}; \hat{s}_t \mid a_t) \geq 0$ by the DPI.
Under optimal prediction, the model's predictive distribution matches the true distribution for each history. As shown in the proof of Proposition~\ref{prop:sufficiency}, this requires the encoder to map histories with different predictive distributions to different latent states. Consequently, $\hat{s}_t$ captures all predictive information, implying $\Delta_I = 0$.
\begin{corollary}[Information-theoretic sufficiency]
\label{cor:info_sufficiency}
Under strict mediation, if the model achieves optimal prediction, then $I(o_{t+1}; h_t \mid \hat{s}_t, a_t) = 0$: the latent state captures all information that the history contains about the next observation.
\end{corollary}
\subsection{Quantitative Relaxation}
\label{app:quantitative}
In practice, no model achieves exactly optimal prediction. We show that the sufficiency gap degrades gracefully: when the model is $\epsilon$-suboptimal in prediction, the predictive information gap is bounded by $\epsilon$.
Let $p(\cdot \mid h_t, a_t)$ denote the true predictive distribution and $p_\theta(\cdot \mid \hat{s}_t, a_t)$ the model's prediction under strict mediation, where $\hat{s}_t$ is the latent state produced from $h_t$. Define the expected prediction gap as
\begin{equation}
\epsilon \;=\; \mathbb{E}_{h_t, a_t}\left[\mathrm{KL}\!\left(p(o_{t+1} \mid h_t, a_t) \;\|\; p_\theta(o_{t+1} \mid \hat{s}_t, a_t)\right)\right],
\end{equation}
where the expectation is over the data distribution of histories and actions, and $\hat{s}_t$ is drawn from the encoder $p_\theta(\cdot \mid h_t)$.
\begin{proposition}[Bounded sufficiency gap]
\label{prop:quantitative}
Under strict mediation, the predictive information gap is bounded by the expected prediction gap:
\begin{equation}
\label{eq:quantitative_bound}
I(o_{t+1}; h_t \mid \hat{s}_t, a_t) \;\leq\; \epsilon.
\end{equation}
\end{proposition}
\begin{proof}
By the chain rule of KL divergence and the strict mediation constraint, we can decompose the expected KL as follows. For a given $\hat{s}_t$ and $a_t$, let $\mathcal{H}(\hat{s}_t)$ denote the set of histories that produce latent state $\hat{s}_t$. The model assigns prediction $p_\theta(o_{t+1} \mid \hat{s}_t, a_t)$ to all histories in $\mathcal{H}(\hat{s}_t)$.
The expected KL divergence can be written as:
\begin{align}
\epsilon &= \mathbb{E}_{h_t, a_t}\left[\mathrm{KL}\!\left(p(o_{t+1} \mid h_t, a_t) \;\|\; p_\theta(o_{t+1} \mid \hat{s}_t, a_t)\right)\right] \\
&= \mathbb{E}_{h_t, a_t}\left[\mathbb{E}_{o_{t+1} \mid h_t, a_t}\left[\log \frac{p(o_{t+1} \mid h_t, a_t)}{p_\theta(o_{t+1} \mid \hat{s}_t, a_t)}\right]\right].
\end{align}
We decompose the log ratio by introducing the conditional distribution $p(o_{t+1} \mid \hat{s}_t, a_t)$, which is the true predictive distribution averaged over histories that produce $\hat{s}_t$:
\begin{equation}
p(o_{t+1} \mid \hat{s}_t, a_t) = \mathbb{E}_{h_t \mid \hat{s}_t, a_t}\left[p(o_{t+1} \mid h_t, a_t)\right].
\end{equation}
Then:
\begin{align}
\epsilon &= \mathbb{E}_{h_t, a_t}\left[\mathrm{KL}\!\left(p(o_{t+1} \mid h_t, a_t) \;\|\; p(o_{t+1} \mid \hat{s}_t, a_t)\right)\right] \nonumber \\
&\quad + \mathbb{E}_{\hat{s}_t, a_t}\left[\mathrm{KL}\!\left(p(o_{t+1} \mid \hat{s}_t, a_t) \;\|\; p_\theta(o_{t+1} \mid \hat{s}_t, a_t)\right)\right]. \label{eq:kl_decomp}
\end{align}
Both terms are nonnegative. The first term equals $I(o_{t+1}; h_t \mid \hat{s}_t, a_t)$, since it measures the expected divergence between the history-specific predictive distribution and the distribution that conditions only on the latent state. Therefore:
\begin{equation}
I(o_{t+1}; h_t \mid \hat{s}_t, a_t) \;\leq\; \epsilon. \qedhere
\end{equation}
\end{proof}
This result provides a practical interpretation of Proposition~\ref{prop:sufficiency}: the latent state need not be an exact sufficient statistic for it to be useful. If the strictly mediated model achieves low prediction error (small $\epsilon$), the latent state is \emph{approximately} sufficient, with the approximation gap bounded by the prediction error itself. This is a property of strict mediation specifically: under leaky architectures, low prediction error places no constraint on $I(o_{t+1}; h_t \mid \hat{s}_t, a_t)$, since the prediction may rely entirely on the history bypass.

\section{Error Propagation Analysis}
\label{app:rollout}
\paragraph{Rollout Error Propagation.}
We compare three rollout schemes for predicting future observations over horizon $k$ given a fixed history $h$ and action sequence $a_{0:k-1}$. The key structural difference lies in what information is passed at each step. \textbf{Stateless} rollout feeds all previous predictions back as inputs, so every past mistake remains present in future inputs and errors accumulate. \textbf{StrictState} rollout compresses history into a single latent state and steps only that forward, yielding a contractive recursion under standard assumptions. \textbf{LeakyState} rollout conditions on both latent state and history at each step; as shown below, this introduces an additional history-feedback term that can amplify errors.
\textbf{(1) \textbf{Stateless} recursive rollout.}
Let $F$ be a history-conditioned predictor. Define
\begin{align}
\hat{o}_{t+1} &= F(h, a_0), \\
\hat{o}_{t+m+1} &= F(h, a_{0:m}, \hat{o}_{t+1:t+m}).
\end{align}
Let $E_m^{\mathrm{no}} = \mathbb{E}[d(\hat{o}_{t+m}, o_{t+m})]$ denote prediction error at step $m$. Assume:
\begin{align}
\text{(i)}\quad & \mathbb{E}[d(F(h,a_{0:m}, o_{t+1:t+m}), o_{t+m+1})] \le \epsilon_H, \\
\text{(ii)}\quad & d(F(h,a_{0:m},\sigma), F(h,a_{0:m},\sigma'))
\le L_H \sum_{i=1}^m d(\sigma_i, \sigma_i').
\end{align}
Then:
\begin{equation}
E_{m+1}^{\mathrm{no}} \;\le\; \epsilon_H + L_H \sum_{i=1}^{m} E_i^{\mathrm{no}}.
\end{equation}
Unrolling:
\begin{equation}
E_m^{\mathrm{no}} \;\le\; \epsilon_H (1 + L_H)^{m-1}.
\end{equation}
\textbf{(2) \textbf{StrictState} rollout.}
Let $\hat{s}_t = \phi(h)$, and define
\begin{align}
\hat{s}_{t+1} &= f(\hat{s}_t, a_0), \\
\hat{s}_{t+m+1} &= f(\hat{s}_{t+m}, a_m), \\
\hat{o}_{t+m} &= g(\hat{s}_{t+m}, a_{m-1}).
\end{align}
Let $Z_m^{\mathrm{strict}} = \mathbb{E}[d_{\mathcal S}(\hat{s}_{t+m}, s_{t+m}^\star)]$ denote latent state error. Assume:
\begin{align}
\text{(i)}\quad & Z_0^{\mathrm{strict}} \le \epsilon_{\mathrm{enc}}, \\
\text{(ii)}\quad & d_{\mathcal S}(f(s,a), f(s',a)) \le L_S d_{\mathcal S}(s,s'), \quad L_S < 1, \\
\text{(iii)}\quad & \mathbb{E}[d_{\mathcal S}(f(s^\star,a), s_{t+1}^\star)] \le \eta.
\end{align}
Then:
\begin{equation}
Z_{m+1}^{\mathrm{strict}} \;\le\; L_S Z_m^{\mathrm{strict}} + \eta.
\end{equation}
Unrolling:
\begin{equation}
Z_m^{\mathrm{strict}} \;\le\; L_S^m \epsilon_{\mathrm{enc}} + \frac{\eta (1 - L_S^m)}{1 - L_S}.
\end{equation}
If $g$ is $L_g$-Lipschitz and has bounded oracle error $\epsilon_{\mathrm{dec}}$:
\begin{equation}
E_m^{\mathrm{strict}} \;\le\; \epsilon_{\mathrm{dec}} + L_g Z_m^{\mathrm{strict}}.
\end{equation}
\paragraph{Remark on discrete latent states.} The error propagation analysis above uses norm-based distances and Lipschitz continuity, which are defined for continuous spaces. When latent states are discrete text strings, as in our instantiation, the appropriate distance metric is a discrete divergence such as edit distance, fact-level F1 complement ($1 - \text{F1}$), or token-level Hamming distance. The structural form of the recurrence which is error propagating through a single state rather than an expanding history, holds under any distance metric satisfying the triangle inequality. The contraction condition $L_S < 1$ generalizes to requiring that the transition function is a contraction mapping under the chosen metric. Our experimental evaluation uses fact-level F1, which serves as the discrete analogue of the norm-based analysis: the bounded rollout error predicted by the theory manifests as slower state.F1 degradation over horizon for StrictState compared to LeakyState (Table~\ref{tab:vs_future}).
\textbf{(3) LeakyState rollout.} Let $\hat{s}_t = \phi(h)$. During rollout, the model conditions on both the latent state and an expanding synthetic history $\hat{h}_{t+m} = (h, a_{0:m}, \hat{o}_{t+1:t+m})$ that incorporates past predicted observations. Define:
\begin{align}
\hat{s}_{t+m+1} &= f_L(\hat{h}_{t+m}, \hat{s}_{t+m}, a_m), \\
\hat{o}_{t+m+1} &= g_L(\hat{h}_{t+m}, \hat{s}_{t+m+1}, a_m).
\end{align}
Because $\hat{h}_{t+m}$ contains all past predicted observations, errors from every previous step remain present in the input and contribute to future predictions, justifying the accumulation term in the error bound.
Let $E_m^{\mathrm{leaky}} = \mathbb{E}[d(\hat{o}_{t+m}, o_{t+m})]$ and $Z_m^{\mathrm{leaky}} = \mathbb{E}[d_{\mathcal S}(\hat{s}_{t+m}, s_{t+m}^\star)]$.
Assume:
\begin{align}
\text{(i)}\quad & \mathbb{E}[d(g_L(h_{t+m-1}^\star, s_{t+m}^\star, a_{m-1}), o_{t+m})] \le \epsilon_L, \\
\text{(ii)}\quad & d(g_L(\tilde{h}, s, a), g_L(\tilde{h}', s', a)) \\
&\quad \le L_H^{\mathrm{leaky}} \sum_{i=1}^{m-1} d(\tilde{o}_i, \tilde{o}_i') + L_g^{\mathrm{leaky}} d_{\mathcal S}(s,s').
\end{align}
Then:
\begin{equation}
E_m^{\mathrm{leaky}} \;\le\; \epsilon_L + L_H^{\mathrm{leaky}} \sum_{i=1}^{m-1} E_i^{\mathrm{leaky}} + L_g^{\mathrm{leaky}} Z_m^{\mathrm{leaky}}.
\end{equation}
Resolving the recurrence:
\begin{align}
E_m^{\mathrm{leaky}}
\;\le\;&
\epsilon_L (1 + L_H^{\mathrm{leaky}})^{m-1}
+ L_g^{\mathrm{leaky}} Z_m^{\mathrm{leaky}} \\
&+ L_H^{\mathrm{leaky}} L_g^{\mathrm{leaky}}
\sum_{j=1}^{m-1}
(1 + L_H^{\mathrm{leaky}})^{m-1-j}
Z_j^{\mathrm{leaky}}.
\end{align}
The final summation is a nonnegative history-feedback term, showing that latent errors at earlier steps can be reintroduced and amplified through the predicted history.
\begin{table}[ht]
\centering
\setlength{\tabcolsep}{3pt}
\renewcommand{\arraystretch}{1.12}
\footnotesize
\caption{Error propagation across rollout schemes. $E_m$: observation error; $Z_m$: latent error. The input column lists the rollout inputs used to advance from time $t+m$ to $t+m+1$, where $\hat{h}_{t+m}=(h,a_{0:m-1},\hat{o}_{t+1:t+m})$ is the synthetic history formed from previous predictions.}
\label{tab:error_comparison}
\begin{tabular}{@{}p{1.55cm}p{3.25cm}p{7.1cm}@{}}
\toprule
Scheme & Rollout input & Bound \\
\midrule
Stateless
&
$h,\; a_{0:m},\; \hat{o}_{t+1:t+m}$
&
$E_m^{\mathrm{no}} \le \epsilon_H(1+L_H)^{m-1}$
\\
\midrule
LeakyState
&
$\hat{h}_{t+m},\; \hat{s}_{t+m},\; a_m$
&
$\begin{aligned}
E_m^{\mathrm{leaky}}
\le\;&
\epsilon_L(1+L_H^{\mathrm{leaky}})^{m-1}
+ L_g^{\mathrm{leaky}} Z_m^{\mathrm{leaky}}
\\
&+
L_H^{\mathrm{leaky}} L_g^{\mathrm{leaky}}
\sum_{j=1}^{m-1}
(1+L_H^{\mathrm{leaky}})^{m-1-j}
Z_j^{\mathrm{leaky}}
\end{aligned}$
\\
\midrule
StrictState
&
$\hat{s}_{t+m},\; a_m$
&
$\begin{aligned}
E_m^{\mathrm{strict}}
\le\;&
\epsilon_{\mathrm{dec}}
+
L_g\!\left(
L_S^m \epsilon_{\mathrm{enc}}
+
\frac{\eta(1-L_S^m)}{1-L_S}
\right)
\end{aligned}$
\\
\bottomrule
\end{tabular}
\end{table}

\section{Comparison with Prior Work}
A comparison with related works are provided in Table~\ref{tab:relatedwork}.
\label{app:related}
\begin{table}[ht]
\centering
\small
\setlength{\tabcolsep}{3pt}
\begin{tabular}{lccccc}
\toprule
 & Obs. & Training & Partial Obs. & Strict Med. & Latent State \\
\midrule
\citet{ha2018world}    & visual  & online  & \checkmark & \checkmark & continuous \\
PlaNet \citep{hafner2019planet}         & visual  & online  & \checkmark & \checkmark & continuous \\
Dreamer V1--V3         & visual  & online  & \checkmark & \checkmark & continuous/discrete \\
Dreamer V4 \citep{hafner2025training}             & visual  & offline & \checkmark & $\times$   & continuous \\
MuZero \citep{schrittwieser2020mastering}               & visual  & online  & \checkmark & \checkmark & continuous \\
PAN \citep{xiang2025pan}                   & visual  & offline & \checkmark & $\times$   & continuous \\
RAP \citep{hao2023reasoning}                   & textual & $\times$& $\times$   & \checkmark & textual \\
\citet{wang2024can}     & textual & $\times$& $\times$   & \checkmark & NA \\
\citet{li2025word}       & textual & offline & \checkmark & $\times$   & NA \\
\midrule
Ours                   & textual & offline & \checkmark & \checkmark & textual \\
\bottomrule
\end{tabular}
\vspace{2pt}
\caption{
Comparison of world modeling approaches along key structural axes.
}
\label{tab:relatedwork}
\end{table}

\section{Dataset Details}
\label{app:data}
\paragraph{Environments and partial observability.}
Both environments define a fully specified symbolic state (entities, attributes, relations, and task variables), but render only a lossy natural-language observation at each step. Observations may omit relevant variables, compress multiple state features into a single description, or vary in surface form without changing semantics. As a result, a single observation is generally insufficient to identify the underlying state, and accurate prediction requires integrating information across time.
\paragraph{Controlling challenge factors.}
To comprehensively probe robustness and abstraction, we vary multiple factors during generation. In \textsc{TextWorld}, we sweep game complexity through world size (number of rooms), task complexity (e.g., recipe length), and skill requirements (e.g., navigation vs.\ cooking/cutting/opening), and generate multiple instances via random seeds (yielding different layouts and object placements). In both \textsc{TextWorld} and \textsc{ScienceWorld}, we additionally inject controlled nuisance variation and noise into actions and/or text (e.g., exploration noise, invalid attempts, lexical perturbations), with step- and trajectory-level metadata recording action provenance and noise type. These knobs let us separately stress partial observability, long-horizon credit assignment, and sensitivity to semantically irrelevant perturbations.
\paragraph{Ground-truth state logging.}
While the model never observes the ground-truth state, we log it for evaluation and diagnostics. At each timestep, we extract the full symbolic simulator state as a set of structured predicate--argument facts (e.g., \texttt{at(entity, room)}, \texttt{in(object, container)}, \texttt{on(object, supporter)}, object attributes such as \texttt{open}, \texttt{closed}, or food states such as \texttt{raw}, \texttt{cooked}). These facts are obtained directly from the game engine’s internal state representation via instrumentation, using the same underlying rules that govern environment dynamics, and are not inferred from observations or constructed manually.
From this full post-action state, we construct a canonical subset of \emph{observation rendering facts}, defined as the minimal set of facts required by the environment’s rendering logic to produce the natural-language observation. This subset is computed deterministically by following the game engine’s observation-generation rules (e.g., Inform7/TextWorld rendering), rather than by heuristic or learned extraction. Concretely, it includes (i) the player’s current location, (ii) all objects present in the current room, (iii) attributes of visible objects (such as food states or container states), (iv) contents of containers when they are open, (v) objects placed on supporters, and (vi) spatial relations and connectivity facts (e.g., exits, doors, and directions) needed to describe the layout. For actions such as \texttt{look} or navigation, this corresponds to the full set of facts required to render the room description; for local actions (e.g., \texttt{take}, \texttt{open}, \texttt{examine}), it includes only the subset of post-action facts explicitly reflected in the resulting observation.
Because both the full state and the observation rendering facts are derived directly from the game engine using its native rules, this construction provides an exact and unambiguous symbolic target aligned with the generated observation text. It therefore enables precise evaluation of predictive sufficiency: state.F1 measures whether the learned latent state captures exactly the information required to reproduce the next observation, independent of surface-form variation. In addition, access to the full logged state enables auxiliary diagnostics, such as verifying transition consistency across equivalent histories and distinguishing true state uncertainty from errors due to textual variation.

\section{Qualitative Example}
\label{app:qualitative}
A illustrative example is provided in Table~\ref{tab:qual_example_full}.
\begin{table*}[!htbp]
\centering
\scriptsize
\begin{adjustbox}{max width=\textwidth,max totalheight=0.9\textheight,center}
\begin{minipage}{\textwidth}
\setlength{\tabcolsep}{5pt}
\renewcommand{\arraystretch}{1.1}
\begin{tabularx}{\textwidth}{@{}>{\raggedright\arraybackslash}X@{}}
\toprule
\textbf{Shared interaction history} \\
\midrule
\begin{minipage}[t]{\linewidth}
\raggedright\ttfamily\scriptsize
-----\\
\# action: look\\
\# obs:\\
-= Bathroom =-\\
Well, here we are in a bathroom. You try to gain information on your surroundings by using a technique you call ``looking.''\\
You can make out a toilet. The toilet is white. But there isn't a thing on it. Sometimes, just sometimes, TextWorld can just be the worst.\\
There is an exit to the east.\\
\textgreater{} -= Bathroom =-0/49\\[0.3em]
-----\\
\# action: examine toilet\\
\# obs:\\
The toilet is solid.\\
\textgreater{} -= Bathroom =-0/50\\[0.3em]
-----\\
\# action: look\\
\# obs:\\
-= Bathroom =-\\
Well, here we are in a bathroom. You try to gain information on your surroundings by using a technique you call ``looking.''\\
You can make out a toilet. The toilet is white. But there isn't a thing on it. Sometimes, just sometimes, TextWorld can just be the worst.\\
There is an exit to the east.\\
\textgreater{} -= Bathroom =-0/51\\[0.3em]
-----\\
\# action: cook knife\\
\# obs:\\
You can't see any such thing.\\
\textgreater{} -= Bathroom =-0/51\\[0.3em]
-----\\
\# action: slice fridge\\
\# obs:\\
You can't see any such thing.\\
\textgreater{} -= Bathroom =-0/51\\[0.3em]
-----\\
\# action: look\\
\# obs:\\
-= Bathroom =-\\
Well, here we are in a bathroom. You try to gain information on your surroundings by using a technique you call ``looking.''\\
You can make out a toilet. The toilet is white. But there isn't a thing on it. Sometimes, just sometimes, TextWorld can just be the worst.\\
There is an exit to the east.\\
\textgreater{} -= Bathroom =-0/52\\[0.3em]
-----\\
\# action: goals\\
\# obs:\\
That's not a verb I recognise.\\
\textgreater{} -= Bathroom =-0/52\\[0.3em]
-----\\
\# action: look\\
\# obs:\\
-= Bathroom =-\\
Well, here we are in a bathroom. You try to gain information on your surroundings by using a technique you call ``looking.''\\
You can make out a toilet. The toilet is white. But there isn't a thing on it. Sometimes, just sometimes, TextWorld can just be the worst.\\
There is an exit to the east.\\
\textgreater{} -= Bathroom =-0/53\\[0.3em]
-----\\
\# action: look\\
\# obs:\\
-= Bathroom =-\\
Well, here we are in a bathroom. You try to gain information on your surroundings by using a technique you call ``looking.''\\
You can make out a toilet. The toilet is white. But there isn't a thing on it. Sometimes, just sometimes, TextWorld can just be the worst.\\
There is an exit to the east.\\
\textgreater{} -= Bathroom =-0/54\\[0.3em]
-----\\
\# action: go east\\
\# obs:\\
-= Corridor =-\\
You've entered a corridor. You begin looking for stuff.\\
There is an open sliding patio door leading south. You need an exit without a door? You should try going north. You don't like doors? Why not try going west, that entranceway is not blocked by one.\\
\textgreater{} -= Corridor =-0/55\\[0.3em]
-----\\
\# action: look
\end{minipage}
\\
\midrule
\multicolumn{1}{@{}c@{}}{\textbf{Ground-truth next observation}}\\
\midrule
\begin{minipage}[t]{\linewidth}
\raggedright\ttfamily\scriptsize
\# obs:\\
-= Corridor =-\\
You've entered a corridor. You begin looking for stuff.\\
There is an open sliding patio door leading south. You need an exit without a door? You should try going north. You don't like doors? Why not try going west, that entranceway is not blocked by one.\\
\textgreater{} -= Corridor =-0/56
\end{minipage}
\\
\bottomrule
\end{tabularx}
\vspace{0.6em}
\begin{tabularx}{\textwidth}{
@{}>{\raggedright\arraybackslash}X
@{\hspace{0.8em}}
>{\raggedright\arraybackslash}X@{}}
\toprule
\textbf{Leaky prediction} & \textbf{Strict prediction} \\
\midrule
\begin{minipage}[t]{\linewidth}
\raggedright\ttfamily\scriptsize
\textbf{\# encoded\_state}\\
{}[at(P, bathroom), carrying(cookbook)]\\[0.4em]
\textbf{\# pred\_obs}\\
-= Bathroom =-\\
Well, here we are in a bathroom. You try to gain information on your surroundings by using a technique you call ``looking.''\\
You can make out a toilet. The toilet is white. But there isn't a thing on it. Sometimes, just sometimes, TextWorld can just be the worst.\\
There is an exit to the east.\\
\textgreater{} -= Bathroom =-0/55
\end{minipage}
&
\begin{minipage}[t]{\linewidth}
\raggedright\ttfamily\scriptsize
\textbf{\# encoded\_state}\\
{}[at(corridor), in(cookbook), on(toilet), fried(cookbook), north\_of(kitchen,corridor), free(corridor,kitchen)]\\[0.4em]
\textbf{\# pred\_obs}\\
-= Corridor =-\\
You enter a corridor. An usual kind of place. Okay, just remember what you're here to do, and everything will go great.\\
As if things weren't amazing enough already, you can even see a light bulb. You idly wonder how they came up with the name TextWorld for this place. It's pretty fitting.\\
There is an exit to the north.\\
\textgreater{} -= Corridor =-0/56
\end{minipage}
\\
\bottomrule
\end{tabularx}
\end{minipage}
\end{adjustbox}
\caption{
Qualitative comparison under identical history. The ground-truth next observation corresponds to entering the corridor.
The leaky model fails to encode sufficient state and predicts a bathroom observation, while the strictly mediated model captures relational structure and predicts the correct transition.
}
\label{tab:qual_example_full}
\end{table*}

\section{Experiment Details}
\label{app:implementation}
\subsection{Fact Extraction from Observations}
\label{app:fact_extraction}
All observation-level rewards and metrics operate on structured fact sets extracted from free-form observation text. Given an observation $o$, we define a deterministic mapping $\mathrm{facts}(o)$ that produces a set of predicate instances (e.g., \texttt{in(apple, fridge)}, \texttt{open(fridge)}).
The extractor operates purely on surface text and does not access hidden simulator states. For each environment, we use a rule-based parser aligned with the observation templates to recover entities, relations, and attributes explicitly stated or strongly implied by the text (e.g., location, containment, object state, and action effects). We additionally apply a small set of deterministic rules (e.g., \texttt{fried(X)} $\Rightarrow$ \texttt{cooked(X)}) to capture logical implications.
While our implementation uses environment-specific parsers, the framework does not depend on simulator internals. In domains without structured templates, $\mathrm{facts}(\cdot)$ can be instantiated using alternative text-to-structure extractors, such as rule-based systems or frozen LLMs.

\subsection{Rewards}
\label{app:reward}
\paragraph{Fact extraction and observation reward.}
All fact-level rewards used in this work are computed from text, not from privileged simulator latent states. Let $\mathrm{facts}(\cdot)$ denote a fact extractor applied either to a free-form observation string or to a generated textual state. For observations, $\mathrm{facts}(o)$ extracts a set of structured facts from the observation text; for latent states, $\mathrm{facts}(s)$ parses the JSON-encoded facts generated by the model. The observation reward compares the predicted observation text $\hat{o}_{t+1}$ with the reference next observation text $o_{t+1}$:
\[
R_{\text{obs}}
=
\operatorname{F1}\bigl(\mathrm{facts}(\hat{o}_{t+1}), \mathrm{facts}(o_{t+1})\bigr).
\]
In our experiments, the observation fact extractor is simulator-aligned in the sense that it follows the observation templates and predicate schema of the benchmark environments. However, it is applied only to surface-level observation strings and generated state strings; it does not query hidden simulator states or privileged environment variables. In environments where such a simulator-aligned extractor is unavailable, the same framework can use any domain-specific text-to-facts extractor, for example a rule-based parser or a frozen LLM extractor calibrated with few-shot labeled observation examples.
\paragraph{State format reward.}
We enforce a lightweight structural constraint on the generated latent state. The state must be a JSON-encoded list of facts, where each fact is a dictionary with exactly two keys, \texttt{"predicate"} and \texttt{"arguments"}. The predicate must be a non-empty string, and the arguments must be a list of non-empty strings. No argument may exceed \texttt{max\_num\_words\_in\_arg} words (set to 8 for ScienceWorld and 3 for TextWorld cooking). Duplicate facts are disallowed, and duplicate arguments within a single fact are not permitted. The format reward $R_{\text{fmt}} \in \{0,1\}$ indicates whether these constraints are satisfied.
\paragraph{Auxiliary state rewards.}
For valid latent states ($R_{\text{fmt}} = 1$), we compute additional rewards to encourage semantic correctness and compactness. These rewards also operate only on extracted facts from text.
\textbf{Support reward.}
The support reward $R_{\text{sup}}$ measures whether facts stated in the predicted next observation $\hat{o}_{t+1}$ are represented in the predicted next latent state $\hat{s}_{t+1}$. We parse both the generated state text and the predicted observation text into fact sets and compute fact-level observation coverage:
\[
\operatorname{Cov}(\hat{s}, \hat{o}) =
\frac{| \text{facts}(\hat{s}) \cap \text{facts}(\hat{o}) |}
{|\text{facts}(\hat{o})|}.
\]
We define
\[
R_{\text{sup}}
=
\operatorname{Cov}(\hat{s}_{t+1}, \hat{o}_{t+1})
+
\bigl(
\operatorname{Cov}(\hat{s}_{t+1}, \hat{o}_{t+1})
-
\operatorname{Cov}(\hat{s}_t, \hat{o}_{t+1})
\bigr),
\]
which rewards the predicted next latent state for containing the facts needed to support the predicted next observation, as well as improvement over the current latent state. Importantly, $R_{\text{sup}}$ does not penalize $\hat{s}_{t+1}$ for containing additional latent facts that are not expressed in the immediate observation, thereby allowing the retention of long-term information.
\textbf{Deduplication reward.}
The deduplication reward $R_{\text{dedup}}$ penalizes redundant or contradictory facts. Each fact is converted into a short natural-language sentence, and all unordered pairs are evaluated using an NLI model\footnote{\url{https://huggingface.co/sileod/deberta-v3-base-tasksource-nli}}. A pair is considered redundant if either the entailment or contradiction score exceeds 0.5. Let $N_{\text{dup}}$ be the number of such pairs and $N_{\text{pairs}}$ the total number of pairs. The reward is
\[
R_{\text{dedup}} = 1 - \frac{N_{\text{dup}}}{N_{\text{pairs}}},
\]
computed for both $s_t$ and $s_{t+1}$ and averaged.
\paragraph{Ablation reward.}
The ablation reward $R_{\text{abl}}$ measures whether individual facts in $\hat{s}_{t+1}$ are necessary for prediction. For each fact, we remove it from the state, decode an observation from the ablated state, and compare it to the observation decoded from the full state. A fact is marked as necessary if its removal reduces the decoded observation fact-level F1 by more than 0.05. The reward is the fraction of necessary facts:
\[
R_{\text{abl}} = \frac{N_{\text{necessary}}}{N_{\text{facts}}}.
\]
This reward is not intended to define sufficiency of the latent state by itself: because the leave-one-out test is local to the sampled transition, a fact may be useful for future or counterfactual actions even if its removal does not affect the decoded observation for the current action. We therefore apply $R_{\text{abl}}$ only at a late stage of training and with a small coefficient, treating it as a weak compactness regularizer. Predictive sufficiency is enforced primarily by strict mediation and the observation reward over diverse action-conditioned transitions. In practice, $R_{\text{abl}}$ biases the model against stale or redundant facts without aggressively pruning persistent state information needed for long-horizon rollout. All leave-one-out decoder evaluations are performed in parallel.
\paragraph{Invalid states.}
If a latent state fails the format constraint ($R_{\text{fmt}} = 0$), auxiliary rewards are not applied. Instead, the total reward is given by the penalized observation reward $R_{\text{obs}} \cdot \gamma_{\text{pen}}$.

\subsection{Models and Training Setup}
We experiment with three pretrained instruction-tuned language models: \texttt{Qwen2.5-0.5B-Instruct} (0.5B parameters), \texttt{Qwen3-4B-Instruct-2507} (4B parameters) and \texttt{Qwen3-32B} (32B parameters). These models serve as the shared backbone for encoding, transition, and decoding, with all stages implemented via prompting as described in Section~4.
Training is performed using factorized GRPO with tree-structured rollouts. Unless otherwise specified, we use a branching factor of $A = B = C = 4$, i.e., each node expands to 4 candidates at each stage (encoding, transition, and decoding).
We train for a single epoch over the offline dataset with learning rate $1\times 10^{-6}$. The reward weighting coefficients are set to $\gamma_{\text{pen}} = 0.5$ and $\lambda_{\text{dedup}} = \lambda_{\text{abl}} = \lambda_{\text{sup}} = 0.1$.
\subsection{Compute}
\label{app:compute}
All experiments are conducted on AWS SageMaker using \texttt{ml.p4de.24xlarge} instances. Each instance is equipped with 8 NVIDIA A100-SXM4 GPUs (80GB memory each), for a total of 640GB GPU memory per node. Training runs utilize all 8 GPUs in parallel. Typical GPU utilization is above 90\%, with per-GPU memory usage around 54GB during training.
Training uses vLLM-backed asynchronous generation together with FSDP-based actor training in verl. We use gradient checkpointing and micro-batching for log-probability computation to fit the 4B model and the expanded fGRPO rollout batches on 8 A100-80GB GPUs.
\textbf{Per-step wall-clock cost.}  Table~\ref{tab:training_cost} reports mean per-step wall-clock time for the three training regimes used in this paper, decomposed into the four phases of the verl PPO loop. Stateless training (SFT) runs only the policy forward-and-backward pass; LeakyState (GRPO) additionally rolls out a single-turn structured response and recomputes log-probabilities under both the current and reference policies for the KL penalty; StrictState-Text (fGRPO) extends this with a three-turn \textsc{Encode} $\rightarrow$ \textsc{Dynamics} $\rightarrow$ \textsc{Decode} rollout, which roughly triples the number of generated tokens per training example.
We implement parallelism for fGRPO (Section~\ref{app:parallel}) to reduce wall-clock time.
The dominant cost driver is output-token volume: both \texttt{gen} and \texttt{update\_actor} scale linearly with response length and branching factor $N$. Strict's three-turn protocol therefore inflates these two stages most. Reference-policy log-probability computation (\texttt{ref}) is also affected because every generated token must be re-scored under the frozen reference. SFT incurs none of these RL-specific phases and is bounded by a single forward-and-backward pass per batch.
\begin{table}[h]
\centering
 \scriptsize
\caption{Mean training cost on 8$\times$A100-80GB at 4B parameters, batch size 8, branching factor $N=4$, and max response length 1024 tokens. fGRPO costs $1.58\times$ as much per step as single-turn GRPO, or 90.0 vs.\ 57.1 GPU-hours per 1k steps.}
\label{tab:training_cost}
\begin{tabular}{lrrrrrr}
\toprule
Method & \texttt{gen} & \texttt{old\_log\_prob} & \texttt{ref} & \texttt{update\_actor} & \textbf{step total} & \textbf{GPU-hr/1k steps} \\
\midrule
Stateless (SFT)             & ---   & ---   & ---   & ---   & {2.28} & 5.1  \\
LeakyState (GRPO)           & 19.9  & 1.0   & 0.9   & 4.0   & {25.7} & 57.1 \\
StrictState-Text (fGRPO)    & 26.2  & 2.0   & 1.8   & 10.5  & {40.5} & 90.0 \\
\bottomrule
\end{tabular}
\end{table}
\subsection{Parallelized factorized GRPO.}
\label{app:parallel}
We implement fGRPO using an asynchronous tree-rollout engine integrated with verl and vLLM. For each training example, the encoder, transition, and decoder stages are expanded as separate prompt calls, but sibling generations from the same parent node are submitted concurrently. Concretely, all $A$ encoder samples, all $B$ transition samples under each encoded state, and all $C$ decoder samples under each predicted next state are launched asynchronously, allowing the rollout tree to be generated in parallel rather than by a serial depth-first traversal. The resulting tree is then flattened into GRPO training samples, with the grouping identifier set to the parent node of each sibling set so that group-relative advantages are computed separately for encoder, transition, and decoder decisions.
This design avoids relying on the standard verl mechanism that repeats prompts before generation. Instead, rollout multiplicity is controlled inside the world-model rollout itself, which is necessary because different nodes in the tree have different prompts and conditioning contexts. We also share request identifiers among siblings generated from the same parent prompt, enabling the rollout backend to batch compatible requests efficiently. For the ablation reward, decoder calls for the full state and all leave-one-fact-out states are likewise submitted concurrently; we skip this computation when it would be identical across decoder-only siblings and therefore cancel out in the GRPO advantage.
\subsection{Prompts}
\label{app:prompts}
\begin{center}
\small
\label{tab:cooking-world-model-history-only-prompt}
\end{center}
\vspace{-0.5em}
\begin{tcolorbox}[promptbox,title={Stateless, System prompt}]
\begin{lstlisting}[style=promptstyle]
You are a world model for a text adventure game.
Given observation history and an action, predict the next observation.
Learn the game mechanics and output format from the history examples.
=== EXAMPLE OBSERVATIONS ===
For navigation (e.g. "go east"):
```
-= Kitchen =-
You arrive in a kitchen. An usual kind of place. Okay, just remember what you're here to do, and everything will go great.
As if things weren't amazing enough already, you can even see a fridge. You idly wonder how they came up with the name TextWorld for this place. It's pretty fitting.
There is an exit to the west.
>                                                                                                                                -= Kitchen =-0/6
```
For taking items (e.g. "take block of cheese from workbench"):
```
You take the block of cheese from the workbench.
>                                                                                                                                -= Shed =-0/23
```
For cooking (e.g. "cook red onion with stove"):
```
You fried the red onion.
Your score has just gone up by one point.
>                                                                                                                                -= Kitchen =-6/35
```
For cutting (e.g. "slice red apple with knife"):
```
You slice the red apple.
Your score has just gone up by one point.
>                                                                                                                                -= Backyard =-5/31
```
For inventory (e.g. "inventory"):
```
You are carrying: a raw yellow potato.
>                                                                                                                                -= Bedroom =-0/16
```
For dropping items (e.g. "drop purple potato"):
```
You drop the purple potato on the ground.
>                                                                                                                                -= Garden =-1/10
```
For preparing meal (e.g. "prepare meal"):
```
Adding the meal to your inventory.
Your score has just gone up by one point.
>                                                                                                                                -= Kitchen =-8/41
```
=== INSTRUCTIONS ===
- Generate observation text matching the game's style from the history
- Learn the output format from the observation history examples
- Output ONLY the observation text, nothing else
\end{lstlisting}
\end{tcolorbox}
\begin{tcolorbox}[promptbox,title={Stateless, User prompt}]
\begin{lstlisting}[style=promptstyle]
=== OBSERVATION HISTORY ===
{
{{ observations[i] | trim }}
Action: {{ actions[i] }}
Result:
{{ observations[i + 1] | trim }}
{
{{ observations[-1] | trim }}
Action: {{ action }}
=== YOUR PREDICTION ===
Predict the observation that will result from taking the action "{{ action }}".
Output ONLY the observation text:
\end{lstlisting}
\end{tcolorbox}
\captionof{table}{Prompts used for the Stateless cooking world model.}

\begin{center}
\small
\label{tab:leaky-cot-cooking-world-model-prompt}
\end{center}
\vspace{-0.5em}
\begin{tcolorbox}[promptbox,title={LeakyState, System prompt}]
\begin{lstlisting}[style=promptstyle]
You are a world model for a TextWorld cooking game.
For each example, do all three steps in one response:
1. Infer the current structured state from the full observation history.
2. Apply the current action to produce the next structured state.
3. Generate the next observation text the player would see.
Use ONLY the predicate vocabulary below.
Location & map:
  at(entity, room)          - "P" (player) or object is in a room
  north_of(room1, room2)    - room1 is north of room2
  south_of(room1, room2)    - room1 is south of room2
  east_of(room1, room2)     - room1 is east of room2
  west_of(room1, room2)     - room1 is west of room2
  free(room1, room2)        - passage between rooms is unblocked
  link(room1, door, room2)  - a door connects two rooms
Containment & placement:
  in(object, container)     - object is inside a container (use "I" for player inventory)
  on(object, supporter)     - object is on a supporter (table, counter, etc.)
Container state:
  open(container)           - container is open
  closed(container)         - container is closed
Object properties:
  sharp(object)             - object is sharp (e.g. knife)
  cookable(food)            - food can be cooked
  cuttable(food)            - food can be cut
  drinkable(food)           - food can be drunk
  edible(food)              - food is safe to eat
  inedible(food)            - food is not safe to eat (needs cooking)
Food state:
  raw(food)                 - food has not been cooked
  cooked(food)              - food has been cooked
  burned(food)              - food has been burned (overcooked)
  needs_cooking(food)       - food must be cooked before eating
  uncut(food)               - food has not been cut
  fried(food)               - food was cooked on a stove
  roasted(food)             - food was cooked in an oven
  grilled(food)             - food was cooked on a BBQ/toaster
  sliced(food)              - food was sliced
  diced(food)               - food was diced
  chopped(food)             - food was chopped
Recipe & meal:
  ingredient_1(food) .. ingredient_5(food) - food is the Nth recipe ingredient
  cooking_location(room, recipe)           - room where the meal must be prepared
  out(meal, recipe)                        - the meal produced by the recipe
  base(food, ingredient)                   - the base food of a recipe ingredient
  used(food)               - food has been used in a meal
  consumed(food)           - food has been eaten/drunk
State format:
- `state` and `next_state` must each be a JSON list of objects.
- Each fact object must be exactly:
  {"predicate": "<name>", "arguments": ["arg1", "arg2"]}
Output format:
- Return exactly one JSON object with these keys:
  - `state`
  - `next_state`
  - `next_observation`
- `state` and `next_state` must be JSON arrays of fact objects.
- `next_observation` must be a plain string.
- Output valid JSON only, with no markdown fences and no extra commentary.
General rules:
- Use only supported predicates.
- Include only facts directly observed or strongly implied by the history.
- Do not hallucinate unsupported objects, locations, or state changes.
- Keep unchanged facts in `next_state` unless the action clearly changes them.
- Make `next_observation` consistent with `next_state` and the action.
- Match TextWorld-style observations.
Few-shot example 1:
Input history:
Observation:
-= Kitchen =-
You are in a kitchen. You see a fridge.
Action:
open fridge
Output:
{
  "state": [
    {"predicate": "at", "arguments": ["P", "kitchen"]},
    {"predicate": "at", "arguments": ["fridge", "kitchen"]},
    {"predicate": "closed", "arguments": ["fridge"]}
  ],
  "next_state": [
    {"predicate": "at", "arguments": ["P", "kitchen"]},
    {"predicate": "at", "arguments": ["fridge", "kitchen"]},
    {"predicate": "open", "arguments": ["fridge"]}
  ],
  "next_observation": "You open the fridge.\n\n>                                                                                                                                -= Kitchen =-"
}
Few-shot example 2:
Input history:
Observation:
-= Kitchen =-
You are in a kitchen. You see a knife on the counter.
Action:
take knife from counter
Output:
{
  "state": [
    {"predicate": "at", "arguments": ["P", "kitchen"]},
    {"predicate": "at", "arguments": ["knife", "kitchen"]},
    {"predicate": "on", "arguments": ["knife", "counter"]},
    {"predicate": "sharp", "arguments": ["knife"]}
  ],
  "next_state": [
    {"predicate": "at", "arguments": ["P", "kitchen"]},
    {"predicate": "in", "arguments": ["knife", "I"]},
    {"predicate": "sharp", "arguments": ["knife"]}
  ],
  "next_observation": "You take the knife from the counter.\n\n>                                                                                                                                -= Kitchen =-"
}
\end{lstlisting}
\end{tcolorbox}
\begin{tcolorbox}[promptbox,title={LeakyState, User prompt}]
\begin{lstlisting}[style=promptstyle]
Infer the current state, apply the action, and generate the next observation.
=== OBSERVATION HISTORY ===
{
Observation:
{{ observations[i] }}
{
Action:
{{ actions[i] }}
{
{
=== CURRENT ACTION ===
{{ action }}
Return the JSON object only.
\end{lstlisting}
\end{tcolorbox}
\captionof{table}{Prompts used for the \texttt{LeakyState} cooking world model.}

\begin{center}
\small
\label{tab:cooking-world-model-prompts}
\end{center}
\vspace{-0.5em}
\begin{tcolorbox}[promptbox,title={StrictState, System prompt}]
\begin{lstlisting}[style=promptstyle]
You are a world model for a text adventure game.
You will simulate this world through a 3-step process:
1. ENCODE: Infer a structured state from observation history
2. DYNAMICS: Update the state based on an action
3. DECODE: Generate the observation text from the state
Learn the game mechanics from the observation history.
Follow instructions carefully at each step.
\end{lstlisting}
\end{tcolorbox}
\begin{tcolorbox}[promptbox,title={StrictState, Encoder prompt}]
\begin{lstlisting}[style=promptstyle]
Now it's the ENCODE step. Based on the observation history provided at the end of this prompt, infer the current world state.
=== PREDICATE VOCABULARY ===
Use ONLY the following predicates. Do NOT invent new ones.
Location & map:
  at(entity, room)          - "P" (player) or object is in a room
  north_of(room1, room2)    - room1 is north of room2
  south_of(room1, room2)    - room1 is south of room2
  east_of(room1, room2)     - room1 is east of room2
  west_of(room1, room2)     - room1 is west of room2
  free(room1, room2)        - passage between rooms is unblocked
  link(room1, door, room2)  - a door connects two rooms
Containment & placement:
  in(object, container)     - object is inside a container (use "I" for player inventory)
  on(object, supporter)     - object is on a supporter (table, counter, etc.)
Container state:
  open(container)           - container is open
  closed(container)         - container is closed
Object properties:
  sharp(object)             - object is sharp (e.g. knife)
  cookable(food)            - food can be cooked
  cuttable(food)            - food can be cut
  drinkable(food)           - food can be drunk
  edible(food)              - food is safe to eat
  inedible(food)            - food is not safe to eat (needs cooking)
Food state:
  raw(food)                 - food has not been cooked
  cooked(food)              - food has been cooked
  burned(food)              - food has been burned (overcooked)
  needs_cooking(food)       - food must be cooked before eating
  uncut(food)               - food has not been cut
  fried(food)               - food was cooked on a stove
  roasted(food)             - food was cooked in an oven
  grilled(food)             - food was cooked on a BBQ/toaster
  sliced(food)              - food was sliced
  diced(food)               - food was diced
  chopped(food)             - food was chopped
Recipe & meal:
  ingredient_1(food) .. ingredient_5(food) - food is the Nth recipe ingredient
  cooking_location(room, recipe)           - room where the meal must be prepared
  out(meal, recipe)                        - the meal produced by the recipe
  base(food, ingredient)                   - the base food of a recipe ingredient
  used(food)               - food has been used in a meal
  consumed(food)           - food has been eaten/drunk
=== STATE FORMAT ===
Infer the state as a JSON list of predicates. Each predicate is an object with:
- "predicate": one of the predicates listed above (string)
- "arguments": list of arguments (strings)
=== EXAMPLE STATE ===
```json
[
  {"predicate": "at", "arguments": ["P", "kitchen"]},
  {"predicate": "in", "arguments": ["banana", "I"]},
  {"predicate": "on", "arguments": ["knife", "counter"]},
  {"predicate": "fried", "arguments": ["banana"]},
  {"predicate": "north_of", "arguments": ["livingroom", "kitchen"]},
  {"predicate": "free", "arguments": ["kitchen", "livingroom"]}
]
```
=== INSTRUCTIONS ===
- Use ONLY predicates from the vocabulary above
- Include only facts you have directly seen or can confidently infer from the history
- Do NOT guess or hallucinate facts not supported by the observations
- It's OK to have an incomplete state - only include what you're confident about
- Output valid JSON only, nothing else
=== OBSERVATION HISTORY ===
{
Observation:
{{ observations[i] }}
{
Action: {{ actions[i] }}
{
{
\end{lstlisting}
\end{tcolorbox}
\begin{tcolorbox}[promptbox,title={StrictState, Transition prompt}]
\begin{lstlisting}[style=promptstyle]
Now it's the DYNAMICS step. Given the current state and action provided at the end of this prompt, predict the next state.
=== PREDICATE VOCABULARY ===
Use ONLY the following predicates. Do NOT invent new ones.
Location & map:
  at(entity, room)          - "P" (player) or object is in a room
  north_of(room1, room2)    - room1 is north of room2
  south_of(room1, room2)    - room1 is south of room2
  east_of(room1, room2)     - room1 is east of room2
  west_of(room1, room2)     - room1 is west of room2
  free(room1, room2)        - passage between rooms is unblocked
  link(room1, door, room2)  - a door connects two rooms
Containment & placement:
  in(object, container)     - object is inside a container (use "I" for player inventory)
  on(object, supporter)     - object is on a supporter (table, counter, etc.)
Container state:
  open(container)           - container is open
  closed(container)         - container is closed
Object properties:
  sharp(object)             - object is sharp (e.g. knife)
  cookable(food)            - food can be cooked
  cuttable(food)            - food can be cut
  drinkable(food)           - food can be drunk
  edible(food)              - food is safe to eat
  inedible(food)            - food is not safe to eat (needs cooking)
Food state:
  raw(food)                 - food has not been cooked
  cooked(food)              - food has been cooked
  burned(food)              - food has been burned (overcooked)
  needs_cooking(food)       - food must be cooked before eating
  uncut(food)               - food has not been cut
  fried(food)               - food was cooked on a stove
  roasted(food)             - food was cooked in an oven
  grilled(food)             - food was cooked on a BBQ/toaster
  sliced(food)              - food was sliced
  diced(food)               - food was diced
  chopped(food)             - food was chopped
Recipe & meal:
  ingredient_1(food) .. ingredient_5(food) - food is the Nth recipe ingredient
  cooking_location(room, recipe)           - room where the meal must be prepared
  out(meal, recipe)                        - the meal produced by the recipe
  base(food, ingredient)                   - the base food of a recipe ingredient
  used(food)               - food has been used in a meal
  consumed(food)           - food has been eaten/drunk
=== INSTRUCTIONS ===
- Apply the action to update the state
- Use ONLY predicates from the vocabulary above
- Update only the predicates affected by the action
- Keep all other predicates unchanged
- Output valid JSON only, nothing else
=== CURRENT STATE ===
{{ state }}
=== ACTION ===
{{ action }}
\end{lstlisting}
\end{tcolorbox}
\begin{tcolorbox}[promptbox,title={StrictState, Decoder prompt}]
\begin{lstlisting}[style=promptstyle]
Now it's the DECODE step. Generate the observation text that the player would see given the state and action provided at the end of this prompt.
=== EXAMPLE OBSERVATIONS ===
For navigation (e.g. "go east"):
```
-= Kitchen =-
You arrive in a kitchen. An usual kind of place. Okay, just remember what you're here to do, and everything will go great.
As if things weren't amazing enough already, you can even see a fridge. You idly wonder how they came up with the name TextWorld for this place. It's pretty fitting.
There is an exit to the west.
>                                                                                                                                -= Kitchen =-0/6
```
For taking items (e.g. "take block of cheese from workbench"):
```
You take the block of cheese from the workbench.
>                                                                                                                                -= Shed =-0/23
```
For cooking (e.g. "cook red onion with stove"):
```
You fried the red onion.
Your score has just gone up by one point.
>                                                                                                                                -= Kitchen =-6/35
```
For cutting (e.g. "slice red apple with knife"):
```
You slice the red apple.
Your score has just gone up by one point.
>                                                                                                                                -= Backyard =-5/31
```
For inventory (e.g. "inventory"):
```
You are carrying: a raw yellow potato.
>                                                                                                                                -= Bedroom =-0/16
```
For dropping items (e.g. "drop purple potato"):
```
You drop the purple potato on the ground.
>                                                                                                                                -= Garden =-1/10
```
For preparing meal (e.g. "prepare meal"):
```
Adding the meal to your inventory.
Your score has just gone up by one point.
>                                                                                                                                -= Kitchen =-8/41
```
=== INSTRUCTIONS ===
- Generate observation text from the state and action taken
- Learn the output format from the observation examples
- Output ONLY the observation text, nothing else
=== STATE ===
{{ state }}
=== ACTION TAKEN ===
{{ action }}
\end{lstlisting}
\end{tcolorbox}
\captionof{table}{Prompts used for the StrictState world model. The system prompt is shared by the encoder, transition, and decoder prompts.}

\end{document}